%% file: main.tex
\documentclass{article}

\usepackage{arxiv}
\usepackage[utf8]{inputenc} 
\usepackage[T1]{fontenc}    
\usepackage{hyperref}       
\usepackage{url}            
\usepackage{booktabs}       
\usepackage{amsfonts}       
\usepackage{nicefrac}       
\usepackage{microtype}      
\usepackage{lipsum}
\usepackage{graphicx}
\graphicspath{ {./images/} }
\input{math_commands.tex}

\usepackage{natbib}
\usepackage{graphicx}
\usepackage{booktabs}
\usepackage{multirow}
\usepackage{amsmath,amssymb,amsfonts}
\usepackage{amsthm}
\usepackage{mathrsfs}
\usepackage[title]{appendix}
\usepackage{xcolor}
\usepackage{textcomp}
\usepackage{manyfoot}
\usepackage{algorithm}
\usepackage{algorithmicx}
\usepackage{algpseudocode}
\usepackage{listings}
\usepackage{float}
\usepackage{tikz}
\usetikzlibrary{positioning, arrows.meta}
\usepackage{adjustbox,lipsum}
\usepackage{silence}
\ErrorsOff*

\newtheorem{theorem}{Theorem}
\newtheorem{remark}{Remark}%

\title{Efficient Mean Curvature Computation on High-Dimensional Data Manifolds}

\author{
 Alexandre Luis Magalh\~aes Levada\\
  Federal University of S\~ao Carlos\\
  13565-905, S\~ao Carlos-SP, Brazil\\
  \texttt{alexandre.levada@ufscar.br} \\
}

\begin{document}
\maketitle
\begin{abstract}
Estimating local mean curvature at each point of a high-dimensional dataset is a key ingredient of geometry-aware machine learning algorithms, such as the Mean Curvature Boundary Points (MCBP) method. The naive implementation of this computation, based on a local shape operator approximated from k-nearest neighbor patches, involves an explicit construction of a matrix $H$ whose trace form yields an $O(m^4)$ cost per point, rendering the approach intractable for datasets with more than a few dozen features. This paper introduces two complementary contributions that together reduce this cost by several orders of magnitude.
The first contribution is an exact algebraic identity. This identity, derived from the orthogonality of the eigenvectors of the covariance matrix and the cyclicity of the trace operator, eliminates $H$ entirely and reduces the per-point cost to $O(m^2)$ after the eigendecomposition. The second contribution addresses the remaining $O(m^3)$ bottleneck of the full eigendecomposition. Since the local covariance matrix has rank at most $k-1 \ll m$, we replace it with a truncated SVD of the $k \times m$ centered data matrix, an $O(k^2 m)$ operation, and derive an analytical approximation for the contribution of the null-space eigenvectors based on the expected value of their outer product under the Haar measure. The resulting estimator has total cost $O(k^2 m + k m p^2)$, where $p = k-1$. Experiments on real-world datasets confirm speedups of 50 to 300 times relative to the original implementation, with negligible loss when the fast estimator is used to replace the original version. By providing a scalable and data-driven estimate of local curvature, the proposed method establishes curvature as a practical geometric feature for a broad range of machine learning tasks, from classical to modern deep learning pipelines.
\end{abstract}

\section{Introduction}
\label{sec:introduction}

The last two decades have witnessed the emergence of \emph{Geometric Machine Learning} (GML) as a unifying paradigm for data analysis, whose central tenet is that the geometric structure of the domain, rather than the vector-space structure of the ambient coordinates, should govern the design of representations, metrics, and algorithms \cite{GML,Papillon2025}. This perspective departs fundamentally from classical Euclidean-based machine learning, which treats every dataset as a collection of points in $\mathbb{R}^{m}$ and relies on the inner product structure of that space for tasks such as similarity search, dimensionality reduction, and classification. The limitations of this flat-geometry assumption became increasingly apparent as practitioners confronted data arising from social networks, brain-imaging graphs, molecular surfaces, protein interaction networks, and high-dimensional sensor arrays, domains in which the relevant notion of similarity or proximity is intrinsically non-Euclidean \citep{bronstein2017geometric}. The seminal position paper by \citet{bronstein2017geometric} crystallised these ideas under the banner of \emph{geometric deep learning}, arguing that the success of convolutional neural networks on images is itself a manifestation of a deeper symmetry principle, translational equivariance, and that a principled extension to graphs, manifolds, and other non-Euclidean domains requires replacing the Euclidean blueprint with one grounded in group theory and differential geometry. This programme was subsequently developed into a comprehensive theoretical framework by \citet{bronstein2021geometric}, who unified convolutional networks, graph neural networks, transformers, and equivariant architectures under a single geometric blueprint inspired by Klein's Erlangen Programme, formalising how symmetry groups (grids, graphs, geodesics, and gauges) determine the admissible learning architectures for each class of structured data.

Underlying this entire research agenda is the \emph{manifold hypothesis} \citep{fefferman2016testing}, which posits that high-dimensional data encountered in practice are not spread uniformly through $\mathbb{R}^{m}$ but are instead concentrated near a smooth, low-dimensional Riemannian manifold $\mathcal{M}$ of intrinsic dimension $d \ll m$ embedded in the ambient space. This hypothesis, formalised with sample-complexity guarantees by \citet{fefferman2016testing} and supported by a large body of empirical evidence, is the theoretical foundation of manifold learning — the family of methods that seek to recover the intrinsic geometry of $\mathcal{M}$ from finite, possibly noisy samples, without access to an explicit parametrisation of the manifold. Landmark algorithms such as ISOMAP \citep{tenenbaum2000global}, Locally Linear Embedding \citep{roweis2000nonlinear}, t-SNE \cite{tsne} and UMAP \cite{umap1,umap2} demonstrated that global geodesic structure and local neighborhood geometry, respectively, can be reliably recovered from point-cloud data, thereby opening the door to a rich family of geometry-aware representations. 

Building on these foundations, GML moves beyond the goal of merely \emph{embedding} data into a lower-dimensional Euclidean space and pursues instead the estimation of differential-geometric quantities, such as tangent spaces, curvature tensors, and shape operators, that characterise the local and global geometry of $\mathcal{M}$ directly from the
data. Local mean curvature, in particular, occupies a privileged role in this programme: it quantifies the degree to which the manifold bends in the ambient space at each point, encoding
information about cluster boundaries, concave and convex geometric features, and low-density transition regions that is invisible to density-based statistics alone \citep{asao2022curvature, cheng2021weingarten}. Efficient and scalable estimation of mean curvature from high-dimensional point clouds is therefore a problem of central importance in GML, yet existing approaches either assume low ambient dimension or incur computational costs that grow polynomially with $m$ in ways that render them intractable for modern datasets.

The present paper addresses this gap by deriving an exact algebraic identity that reduces the per-point cost of the mean curvature estimator introduced in \citet{mcbp2025} from $O(m^{4})$ to $O(m^{3})$, and by combining this identity with a truncated singular value decomposition to achieve a further reduction to $O(k^{2}m)$ for datasets in which the neighborhood size $k$ is small relative to the ambient dimension $m$, precisely the regime that characterises
high-dimensional machine learning applications.

The present paper makes three interrelated scientific contributions that together constitute a significant step forward in the computational foundations of Geometric Machine Learning. First, we establish an exact closed-form identity that reformulates the pointwise mean curvature estimator of \citet{mcbp2025}, originally defined through an explicit feature matrix $H \in \mathbb{R}^{m \times (m + \binom{m}{2})}$ involving all element-wise squares and pairwise products of local eigenvectors, as a compact expression involving only the $m \times m$ matrix $C = W^{\top}W^{(2)}$, where $W$ is the orthogonal matrix of eigenvectors of the local covariance and $W^{(2)}$ is its element-wise square. This identity, whose derivation exploits the orthogonality of $W$ to collapse a quartic tensor contraction into a single matrix multiplication followed by element-wise operations, reduces the dominant per-point computational cost from $O(m^{4})$ to $O(m^{3})$ with no approximation error.

Second, we show that, in the high-dimensional regime where $m \gg k$, the local covariance matrix has numerical rank at most $p = k - 1$, and we exploit this low-rank structure to replace the $O(m^{3})$ full eigendecomposition with a truncated singular value
decomposition of the $k \times m$ centred neighborhood matrix, at cost $O(k^{2}m)$. The contribution of the remaining $m - p$ null-space eigenvectors is handled via an analytical approximation grounded in the expected outer product of random orthonormal bases under the Haar measure, yielding a total per-point cost of $O(k^{2}m + kmp^{2})$.

Third, we demonstrate through systematic experiments that the resulting estimator achieves speedups of up to $\mathbf{800\times}$ over the original implementation for $m = 200$, while preserving the geometric fidelity required for downstream tasks, making high-dimensional mean curvature estimation, for the first time, computationally viable as a routine preprocessing step in modern machine learning pipelines.

Beyond the MCBP algorithm that motivated this work, the efficient estimation of local mean curvature opens new possibilities across a broad spectrum of application domains in which the intrinsic geometry of the data manifold carries information that density-based or distance-based statistics cannot capture. In \textit{unsupervised clustering and boundary detection}, high-curvature regions of the data manifold naturally identify transition zones between clusters, low-density interfaces, and geometric irregularities that mark the boundaries of data partitions~\citep{mcbp2025}; incorporating curvature as a feature in algorithms such as DBSCAN or HDBSCAN can therefore sharpen cluster boundaries and reduce sensitivity to bandwidth parameters in settings with non-linear or heterogeneous cluster shapes.

In \textit{anomaly and outlier detection}, points at which the local manifold exhibits unusually high or inconsistent curvature are natural candidates for anomalous observations, since they correspond to regions where the data departs from the smooth low-dimensional structure expected under the manifold hypothesis~\citep{fefferman2016testing}; this geometric
perspective on anomaly scoring complements, and in many cases outperforms, purely density-based criteria, as recently demonstrated for graph-structured data \citet{Curvgad} and multi-class unsupervised anomaly detection \cite{Dinomaly}. 

In \textit{semi-supervised and active learning}, identifying high-curvature boundary points provides a principled, geometry-driven criterion for selecting the most informative samples to label, since the classification uncertainty of discriminative models is highest precisely near the curved boundaries between classes~\citep{mcbp2025}; this yields a curvature-aware
query strategy that is complementary to conventional uncertainty-sampling or margin-based approaches. In \textit{single-cell genomics and bioinformatics}, where datasets routinely have thousands of features ($m \sim 10^{3}$--$10^{4}$) but lie near manifolds of intrinsic dimension of order tens to hundreds, local curvature can reveal differentiation trajectories,
identify transitional cell states near bifurcation points of the developmental manifold, and flag cells at the geometric boundary between annotated cell types, tasks for which current
dimensionality reduction pipelines based on PCA, UMAP, or $t$-SNE provide only indirect and nonparametric evidence~\citep{imoto2022scrna}. In tabular data, autoencoder-based semi-supervised learning architectures have shown results comparable to state-of-the-art methods, particularly in scenarios with very limited labeled data \cite{SSLAE}, indicating that curvature-aware autoencoders and active learning methods \cite{AL2023} are a promising natural direction of evolution. 

In \textit{geometric deep learning on graphs and point clouds}, node-level or point-level curvature estimates provide expressive geometric descriptors that can be incorporated as input features or regularisation signals in graph neural networks, in a manner analogous to how discrete Ricci curvature has been used to detect community structure, improve message passing, and characterise anomalous nodes in attributed networks~\citep{Curvgad, cheng2021weingarten}. The computational advances reported here, by making mean curvature
estimation tractable for the ambient dimensions encountered in these domains, remove a key bottleneck that has so far prevented the wider adoption of curvature-aware methods in high-dimensional machine learning.

The remaining of the paper is organized as follows: Section 2 presents the theoretical background on differential geometry. Section 3 discusses the original local mean curvature estimation algorithm, which is computationally unfeasible for large and high dimensional datasets. Section 4 describes the proposed algebraic identity for efficient local mean curvature computation and the resulting algorithm. Section 5 shows the computational experiments and the obtained results. Finally, Section 6 presents the conclusions and final remarks. 

\section{Differential Geometry Basics}

The mathematical backbone of the proposed method is drawn from classical differential geometry, the branch of mathematics concerned with the infinitesimal structure of smooth curves, surfaces, and their higher-dimensional generalisations, Riemannian manifolds.
Rather than treating a dataset as an unstructured collection of points in $\mathbb{R}^{m}$, we adopt the perspective, now standard in Geometric Machine Learning, that observations are discrete samples from an underlying smooth manifold $\mathcal{M}$ embedded in the ambient
space, and that the most informative geometric signal resides not in pairwise distances alone but in the way $\mathcal{M}$ curves and stretches relative to its embedding
\citep{spivak1999comprehensive, do_carmo_differential_2016,	ONeill2006, tu2017differential}.
To make this precise, we rely on five classical objects. The \emph{tangent space} $T_{p}\mathcal{M}$ at a point $p \in \mathcal{M}$ is the best linear approximation to the manifold at $p$, encoding the directions along which one can move while remaining, to first order, on $\mathcal{M}$. The \emph{first fundamental form}, or metric tensor $g$, is the
restriction of the ambient inner product to each tangent space; it governs intrinsic measurements (lengths, angles, and areas) that are independent of how $\mathcal{M}$ sits in $\mathbb{R}^{m}$. The \emph{second fundamental form} $\mathcal{I\!I}$ captures the
complementary, extrinsic information: it measures how much tangent vectors twist out of the tangent space as one moves along $\mathcal{M}$, thereby quantifying the bending of the manifold in the ambient space \citep{do_carmo_differential_2016, DGApp}. From these two forms one constructs the \emph{shape operator} $\mathcal{S}$, the self-adjoint linear map from $T_{p}\mathcal{M}$ to itself obtained by composing $\mathcal{I\!I}$ with the inverse of
$g$, whose eigenvalues are the principal curvatures $\kappa_{1}, \ldots, \kappa_{d}$ of $\mathcal{M}$ at $p$. Their arithmetic mean defines the \emph{mean curvature} $H = d^{-1}\sum_{i=1}^{d} \kappa_{i}$, the scalar quantity that aggregates the net bending of the manifold at each point into a single, geometrically interpretable value \citep{spivak1999comprehensive, needham2021visual, ONeill2006}. In the data-analytic context that motivates this paper, mean curvature serves as a principled, nonparametric indicator of geometric irregularity: regions of $\mathcal{M}$ where $H$ is elevated correspond to sharp transitions between dense clusters, concave or convex geometric features, and low-density interfaces, precisely the boundary structures that classical density-based statistics tend to
conflate with outliers or noise \citep{mcbp2025}. The remainder of this section formalises these objects in the notation used throughout the paper, with an emphasis on the discrete,
sample-based estimators through which they become computationally accessible.

\subsection{Tangent Spaces}
\label{subsec:tangent_spaces}

The first step toward a rigorous treatment of curvature is to formalise the notion of \emph{direction} on a manifold. Unlike points in $\mathbb{R}^{m}$, whose directions are globally defined by the standard basis, a manifold $\mathcal{M}$ need not inherit a canonical linear structure from the ambient space. The tangent space resolves this difficulty by associating to each point $p \in \mathcal{M}$ a vector space that captures, to first
order, the local geometry of $\mathcal{M}$ near $p$ \citep{do_carmo_differential_2016, tu2017differential}.

\paragraph{Intuitive description.}
Consider a smooth curve $\boldsymbol{\alpha} : (-\varepsilon, \varepsilon) \to \mathcal{M}$ with $\boldsymbol{\alpha}(0) = p$. The velocity vector $\boldsymbol{\alpha}'(0) \in \mathbb{R}^{m}$ points in a direction that is tangent to $\mathcal{M}$ at $p$. The collection of all such velocity vectors, as $\boldsymbol{\alpha}$ ranges over all smooth curves through $p$, forms a $d$-dimensional linear subspace of $\mathbb{R}^{m}$, the tangent space at $p$.
Intuitively, $T_{p}\mathcal{M}$ is the best flat (affine) approximation to $\mathcal{M}$ at $p$: it is the unique $d$-dimensional hyperplane that osculates $\mathcal{M}$ to first order at that point \citep{ONeill2006, needham2021visual}. Figure \ref{fig:tangent} illustrates the tangent plane to a sphere at a given point.

\begin{figure}
	\centering
	\includegraphics[scale=0.2]{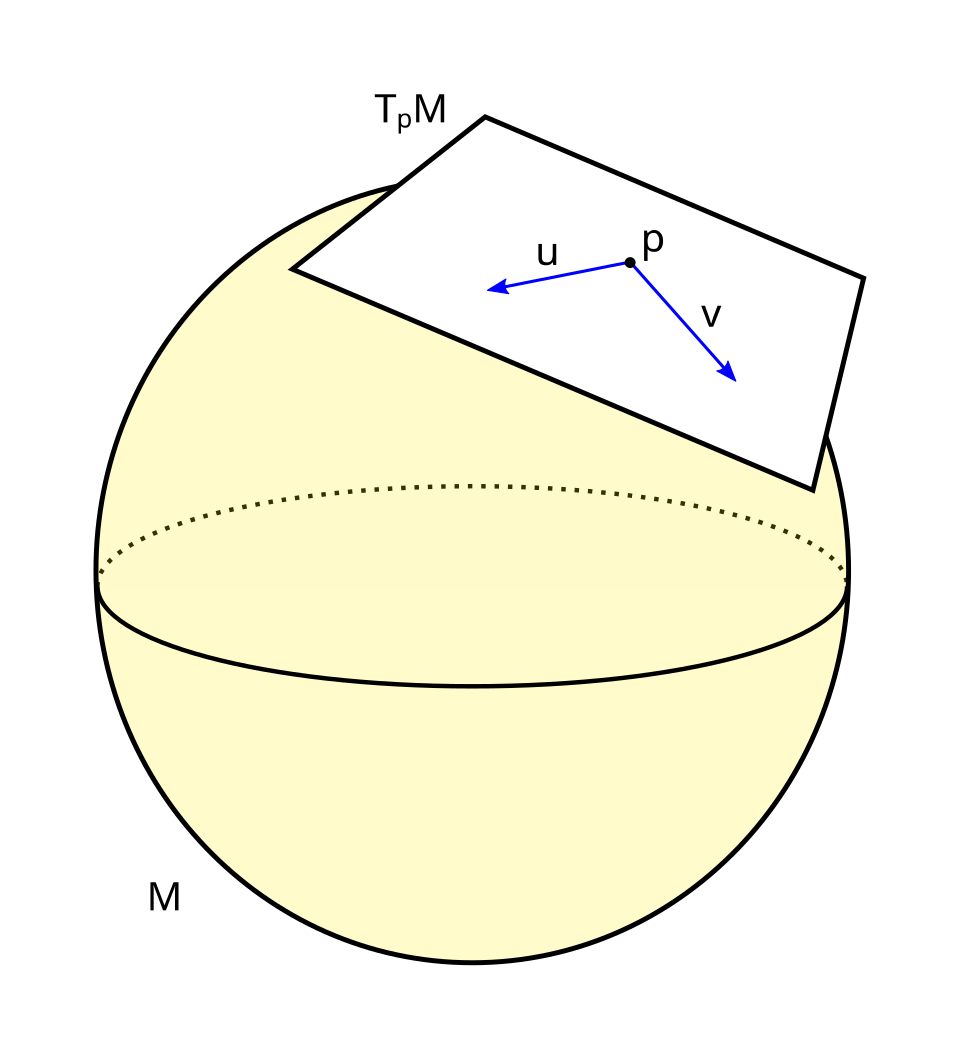}
	\caption{Illustration of the tangent space at a point $\mathbf{x}$ on a sphere embedded in $\mathbb{R}^3$. The tangent plane $T_{\mathbf{x}}\mathcal{M}$ provides a first-order linear approximation of the manifold in a local neighborhood of $\mathbf{x}$, capturing the directions of admissible infinitesimal variations along the surface.}
	\label{fig:tangent}
\end{figure} 

\paragraph{Formal definition.}
Let $\mathcal{M}$ be a smooth $d$-dimensional manifold embedded in $\mathbb{R}^{m}$, and let $p \in \mathcal{M}$. The \emph{tangent space} of $\mathcal{M}$ at $p$ is defined as

\begin{equation}
	T_{p}\mathcal{M}
	\;=\;
	\bigl\{
	\boldsymbol{\alpha}'(0) \in \mathbb{R}^{m}
	\;\big|\;
	\boldsymbol{\alpha} : (-\varepsilon, \varepsilon) \to \mathcal{M}
	\text{ smooth},\;
	\boldsymbol{\alpha}(0) = p
	\bigr\}.
	\label{eq:tangent_def}
\end{equation}

One verifies that $T_{p}\mathcal{M}$ is closed under addition and scalar multiplication, and hence is indeed a vector space of dimension $d$ \citep{tu2017differential}. Elements of $T_{p}\mathcal{M}$ are called \emph{tangent vectors} at $p$.

\paragraph{Coordinate representation.}
Let $\boldsymbol{\phi} : U \subset \mathbb{R}^{d} \to \mathcal{M}$ be a local parametrisation of $\mathcal{M}$ around $p$, with $\boldsymbol{\phi}(\mathbf{u}_{0}) = p$ for some $\mathbf{u}_{0} \in U$. The partial derivatives of $\boldsymbol{\phi}$ at $\mathbf{u}_{0}$,

\begin{equation}
	\partial_{i}
	\;\equiv\;
	\frac{\partial \boldsymbol{\phi}}{\partial u^{i}}
	\bigg|_{\mathbf{u}_{0}}
	\;\in\; \mathbb{R}^{m},
	\qquad i = 1, \ldots, d,
	\label{eq:coord_basis}
\end{equation}

form a basis for $T_{p}\mathcal{M}$, so that every tangent vector $\mathbf{v} \in T_{p}\mathcal{M}$ can be written as

\begin{equation}
	\mathbf{v}
	\;=\;
	\sum_{i=1}^{d} v^{i}\, \partial_{i},
	\qquad v^{i} \in \mathbb{R}.
	\label{eq:tangent_coords}
\end{equation}

The matrix $J_{\boldsymbol{\phi}} = [\partial_{1}\mid\cdots\mid\partial_{d}] \in \mathbb{R}^{m \times d}$ is the Jacobian of the parametrisation; its column space is precisely
$T_{p}\mathcal{M}$ \citep{do_carmo_differential_2016, DGApp}.

\paragraph{The tangent bundle.}
Varying $p$ over all of $\mathcal{M}$ yields the \emph{tangent bundle}

\begin{equation}
	T\mathcal{M}
	\;=\;
	\bigsqcup_{p \in \mathcal{M}} T_{p}\mathcal{M}
	\;=\;
	\bigl\{(p, \mathbf{v})
	\;\big|\;
	p \in \mathcal{M},\;
	\mathbf{v} \in T_{p}\mathcal{M}
	\bigr\},
	\label{eq:tangent_bundle}
\end{equation}

which is itself a smooth manifold of dimension $2d$. A \emph{vector field} on $\mathcal{M}$ is a smooth section of $T\mathcal{M}$, \textit{i.e.}, a smooth map $X : \mathcal{M} \to T\mathcal{M}$ with $X(p) \in T_{p}\mathcal{M}$ for every $p$. Vector fields serve as the natural domain of differential operators such as the covariant derivative and the Lie bracket, which in turn underpin the definitions of the curvature tensors introduced in
subsequent subsections \citep{spivak1999comprehensive, ONeill2006}.

\paragraph{Normal space.}
The orthogonal complement of $T_{p}\mathcal{M}$ in $\mathbb{R}^{m}$ is the \emph{normal space} at $p$,

\begin{equation}
	N_{p}\mathcal{M}
	\;=\;
	\bigl\{
	\mathbf{n} \in \mathbb{R}^{m}
	\;\big|\;
	\langle \mathbf{n}, \mathbf{v} \rangle = 0
	\;\;\forall\, \mathbf{v} \in T_{p}\mathcal{M}
	\bigr\},
	\label{eq:normal_space}
\end{equation}

of dimension $m - d$. The ambient space decomposes orthogonally as $\mathbb{R}^{m} = T_{p}\mathcal{M} \oplus N_{p}\mathcal{M}$, and the unit vectors in $N_{p}\mathcal{M}$ are the \emph{normals} to $\mathcal{M}$ at $p$. For a hypersurface ($d = m-1$), the normal space is one-dimensional and spanned by a single unit normal $\mathbf{n}(p)$; for co-dimension greater than one, there exists a family of normal directions, and the extrinsic curvature depends on the chosen normal \citep{do_carmo_differential_2016, spivak1999comprehensive}.

\paragraph{Relevance in the data-analytic setting.}
In the context of machine learning, the tangent space plays a twofold role. First, it provides the local linear model underlying a wide family of geometry-aware algorithms: principal component analysis (PCA), locally linear embedding \citep{roweis2000nonlinear}, and diffusion maps all implicitly estimate $T_{p}\mathcal{M}$ through the leading eigenvectors of local covariance or graph-Laplacian operators. Second, and most directly relevant to the present work, the tangent space is the domain on which the shape operator $\mathcal{S}$
acts and on which the second fundamental form $\mathrm{I\!I}$ is defined; it is therefore the indispensable geometric substrate for any curvature-based analysis. Given a local neighborhood $\mathcal{N}(p)$ of $p$ in the dataset, we estimate $T_{p}\mathcal{M}$ as the span of the $d$ leading eigenvectors of the sample covariance matrix $\boldsymbol{\Sigma}_{p}$ (defined in Section~\ref{subsec:setup}), a procedure whose consistency under the manifold hypothesis is established in \citet{Singer2012} and whose sample complexity is controlled by the local reach and curvature of $\mathcal{M}$. This eigenvector-based estimate of the tangent space is not only the starting point for our curvature estimator but also the object whose orthogonality, formally expressed as $W^{\!\top}W = I_{m}$, drives the key algebraic simplification derived in Section~\ref{sec:identity}.

\subsection{First Fundamental Form}
\label{subsubsec:first_ff}

Having established the tangent space $T_{p}\mathcal{M}$ as the natural linear approximation to the manifold at each point, the next step is to equip it with a notion of \emph{length} and \emph{angle}. This is the role of the first fundamental form, the foundational
object of intrinsic Riemannian geometry: it encodes how distances and angles on $\mathcal{M}$ relate to those of the ambient space $\mathbb{R}^{m}$, without reference to any particular embedding \citep{do_carmo_differential_2016, spivak1999comprehensive}.

\paragraph{Definition.}
Let $\mathcal{M}$ be a smooth $d$-dimensional manifold embedded in $\mathbb{R}^{m}$, and let $p \in \mathcal{M}$. The \emph{first fundamental form} at $p$, also called the \emph{metric tensor} or \emph{Riemannian metric}, is the bilinear form

\begin{equation}
	g_{p} : T_{p}\mathcal{M} \times T_{p}\mathcal{M} \longrightarrow \mathbb{R},
	\qquad
	g_{p}(\mathbf{u}, \mathbf{v})
	\;=\;
	\langle \mathbf{u},\, \mathbf{v} \rangle_{\mathbb{R}^{m}},
	\label{eq:metric_def}
\end{equation}

obtained by restricting the standard Euclidean inner product of $\mathbb{R}^{m}$ to the tangent space $T_{p}\mathcal{M}$. The map $p \mapsto g_{p}$ is required to be smooth, yielding a smooth $(0,2)$-tensor field on $\mathcal{M}$ \citep{tu2017differential}. By construction, $g_{p}$ is symmetric, 

\begin{equation}
	g_{p}(\mathbf{u}, \mathbf{v})
	\;=\;
	g_{p}(\mathbf{v}, \mathbf{u}),
	\qquad
	\forall\, \mathbf{u}, \mathbf{v} \in T_{p}\mathcal{M},
	\label{eq:metric_sym}
\end{equation}

and positive definite,

\begin{equation}
	g_{p}(\mathbf{v}, \mathbf{v}) \;\geq\; 0,
	\qquad
	g_{p}(\mathbf{v}, \mathbf{v}) = 0
	\;\Longleftrightarrow\;
	\mathbf{v} = \mathbf{0}.
	\label{eq:metric_pd}
\end{equation}

\paragraph{Matrix representation in local coordinates.}
Let $\boldsymbol{\phi} : U \subset \mathbb{R}^{d} \to \mathcal{M}$ be a local parametrisation around $p$, with coordinate basis $\{\partial_{1}, \ldots, \partial_{d}\}$ as defined in
\eqref{eq:coord_basis}. The metric tensor is completely determined by its values on pairs of
basis vectors, which define the \emph{Gram matrix} $G = (g_{ij}) \in \mathbb{R}^{d \times d}$ with entries

\begin{equation}
	g_{ij}
	\;=\;
	g_{p}(\partial_{i},\, \partial_{j})
	\;=\;
	\left\langle
	\frac{\partial \boldsymbol{\phi}}{\partial u^{i}},\,
	\frac{\partial \boldsymbol{\phi}}{\partial u^{j}}
	\right\rangle_{\!\mathbb{R}^{m}},
	\qquad
	i, j = 1, \ldots, d.
	\label{eq:gram_entries}
\end{equation}

In matrix form, denoting the Jacobian of the parametrisation by $J_{\boldsymbol{\phi}} = [\partial_{1} \mid \cdots \mid \partial_{d}] \in \mathbb{R}^{m \times d}$, we have the compact expression

\begin{equation}
	G
	\;=\;
	J_{\boldsymbol{\phi}}^{\top} J_{\boldsymbol{\phi}}
	\;\in\; \mathbb{R}^{d \times d}.
	\label{eq:gram_matrix}
\end{equation}

An arbitrary tangent vector $\mathbf{v} = \sum_{i} v^{i}\partial_{i}$ has squared norm

\begin{equation}
	\|\mathbf{v}\|^{2}_{g}
	\;=\;
	g_{p}(\mathbf{v}, \mathbf{v})
	\;=\;
	\sum_{i,j=1}^{d} g_{ij}\, v^{i} v^{j}
	\;=\;
	\mathbf{v}^{\top} G\, \mathbf{v},
	\label{eq:norm_g}
\end{equation}

and the angle $\theta$ between two tangent vectors $\mathbf{u}$ and $\mathbf{v}$ is given by

\begin{equation}
	\cos\theta
	\;=\;
	\frac{g_{p}(\mathbf{u}, \mathbf{v})}
	{\|\mathbf{u}\|_{g}\,\|\mathbf{v}\|_{g}}
	\;=\;
	\frac{\mathbf{u}^{\top} G\, \mathbf{v}}
	{\sqrt{\mathbf{u}^{\top} G\, \mathbf{u}}\,
		\sqrt{\mathbf{v}^{\top} G\, \mathbf{v}}}.
	\label{eq:angle_g}
\end{equation}

\paragraph{Intrinsic geometric quantities.}
The metric tensor governs three fundamental measurements on $\mathcal{M}$, all of which are intrinsic, that is, independent of the particular embedding in $\mathbb{R}^{m}$.

The \emph{arc length} of a smooth curve $\boldsymbol{\alpha} : [a, b] \to \mathcal{M}$ is

\begin{equation}
	\mathcal{L}(\boldsymbol{\alpha})
	\;=\;
	\int_{a}^{b}
	\sqrt{g_{\boldsymbol{\alpha}(t)}
		\!\bigl(\boldsymbol{\alpha}'(t),\, \boldsymbol{\alpha}'(t)\bigr)}
	\;\mathrm{d}t
	\;=\;
	\int_{a}^{b}
	\bigl\|\boldsymbol{\alpha}'(t)\bigr\|_{\mathbb{R}^{m}}
	\;\mathrm{d}t.
	\label{eq:arc_length}
\end{equation}

The \emph{volume element} on $\mathcal{M}$, used to integrate scalar fields over the manifold, is

\begin{equation}
	\mathrm{d}V_{\mathcal{M}}
	\;=\;
	\sqrt{\det G}\;\mathrm{d}u^{1}\cdots\mathrm{d}u^{d},
	\label{eq:vol_element}
\end{equation}

where $\det G > 0$ by positive definiteness of $g$.
The \emph{geodesic distance} between two points $p, q \in \mathcal{M}$
is the infimum of arc lengths over all smooth curves connecting them:

\begin{equation}
	d_{\mathcal{M}}(p, q)
	\;=\;
	\inf_{\boldsymbol{\alpha}:\, \boldsymbol{\alpha}(a)=p,\,
		\boldsymbol{\alpha}(b)=q}
	\mathcal{L}(\boldsymbol{\alpha}).
	\label{eq:geodesic_dist}
\end{equation}

Geodesic distance, rather than Euclidean distance in $\mathbb{R}^{m}$, is the natural notion of proximity on $\mathcal{M}$ and underlies algorithms such as ISOMAP \citep{tenenbaum2000global}.

\paragraph{Relationship to the covariance matrix.}
In the discrete, data-driven setting, the first fundamental form admits a natural sample-based estimator. Given the centred neighborhood matrix $X_{c} \in \mathbb{R}^{k \times m}$ (whose rows are the centred vectors $\mathbf{x}_{i_{j}} - \bar{\mathbf{x}}$), the sample covariance matrix

\begin{equation}
	\boldsymbol{\Sigma}
	\;=\;
	\frac{1}{k-1}\,X_{c}^{\top} X_{c}
	\;\in\; \mathbb{R}^{m \times m}
	\label{eq:sample_cov}
\end{equation}

can be interpreted as a discrete approximation of the metric tensor pulled back to the ambient coordinates. More precisely, if $V \in \mathbb{R}^{m \times d}$ is the matrix whose
columns are the $d$ leading eigenvectors of $\boldsymbol{\Sigma}$ (providing an estimate of the tangent basis $\{\partial_{i}\}$), then the estimated Gram matrix is

\begin{equation}
	\widehat{G}
	\;=\;
	V^{\top} \boldsymbol{\Sigma}\, V
	\;=\;
	\operatorname{diag}(\lambda_{1}, \ldots, \lambda_{d}),
	\label{eq:estimated_gram}
\end{equation}

where $\lambda_{1} \geq \cdots \geq \lambda_{d}$ are the leading eigenvalues of $\boldsymbol{\Sigma}$. This diagonal form reflects the fact that the PCA coordinate system
diagonalises the metric, so that the estimated tangent directions are locally orthonormal up to the scale factors $\lambda_{i}$ \citep{do_carmo_differential_2016, DGApp}.

\paragraph{Relevance in the proposed method.}
The first fundamental form plays a twofold role in our curvature estimator. On the one hand, $\boldsymbol{\Sigma}$, whose spectral decomposition drives both the exact and the fast modes of Algorithm~1, encodes precisely the sample-based metric of the local neighborhood, so
that the full eigendecomposition $\boldsymbol{\Sigma} = W \operatorname{diag}(\mathbf{v}) W^{\top}$ simultaneously estimates the tangent frame ($W$) and the local scale factors ($\mathbf{v}$). On the other hand, the shape operator $\mathcal{S}$ (shape operator), requires the metric tensor through the relation $\mathcal{S} = G^{-1}L$, where $L$ is the matrix of the second fundamental form; in the eigenvector basis, $G$ becomes the identity matrix (since PCA yields an orthonormal frame), which is precisely the simplification that makes the algebraic identity of Theorem~\ref{thm:identity} possible \citep{ONeill2006, cheng2021weingarten}.

\subsection{Second Fundamental Form}
\label{subsubsec:second_ff}

While the first fundamental form captures the \emph{intrinsic} geometry of $\mathcal{M}$, lengths, angles, and areas that can be measured without leaving the manifold, it is entirely blind to how $\mathcal{M}$ bends in the ambient space $\mathbb{R}^{m}$. Two surfaces with identical metric tensors may have radically different shapes: a flat plane and a cylinder, for instance, are locally isometric yet geometrically distinct objects in $\mathbb{R}^{3}$.
The \emph{second fundamental form} is the fundamental extrinsic object that resolves this ambiguity: it encodes, at each point $p \in \mathcal{M}$, the rate at which the manifold departs from its tangent hyperplane, thereby quantifying how $\mathcal{M}$ curves relative to its embedding \citep{do_carmo_differential_2016, spivak1999comprehensive, ONeill2006}.

\paragraph{Motivation via normal curvature.}
Consider a smooth curve $\boldsymbol{\alpha} : (-\varepsilon,\varepsilon) \to \mathcal{M}$ with $\boldsymbol{\alpha}(0) = p$ and unit tangent vector $\mathbf{t} =  \boldsymbol{\alpha}'(0) \in T_{p}\mathcal{M}$. The acceleration $\boldsymbol{\alpha}''(0) \in \mathbb{R}^{m}$ measures how rapidly the curve bends in the ambient space. Decomposing this acceleration into its tangential and normal components with respect to $\mathcal{M}$,

\begin{equation}
	\boldsymbol{\alpha}''(0)
	\;=\;
	\underbrace{
		\bigl(\boldsymbol{\alpha}''(0)\bigr)^{\!\top}
	}_{\text{tangential}}
	+\;
	\underbrace{
		\bigl(\boldsymbol{\alpha}''(0)\bigr)^{\!\perp}
	}_{\text{normal}},
	\label{eq:accel_decomp}
\end{equation}

the normal component measures the bending of $\mathcal{M}$ itself in the direction $\mathbf{t}$, independently of how the curve is parametrised within $\mathcal{M}$.
The \emph{normal curvature} in the direction $\mathbf{t}$ is defined as

\begin{equation}
	\kappa_{\mathbf{n}}(\mathbf{t})
	\;=\;
	\bigl\langle \boldsymbol{\alpha}''(0),\, \mathbf{n}(p)
	\bigr\rangle_{\mathbb{R}^{m}},
	\label{eq:normal_curv}
\end{equation}

where $\mathbf{n}(p)$ is a unit normal to $\mathcal{M}$ at $p$. One can show that $\kappa_{\mathbf{n}}(\mathbf{t})$ depends only on the direction $\mathbf{t}$, not on the particular curve $\boldsymbol{\alpha}$ chosen to represent it, a result known as Meusnier's theorem \citep{do_carmo_differential_2016}. The second fundamental form systematises these normal curvature values into a single bilinear object.

\paragraph{Definition.}
Let $\mathbf{n} : \mathcal{M} \to \mathbb{R}^{m}$ be a smooth unit normal field on $\mathcal{M}$, and let $D_{\mathbf{u}}\,\mathbf{n}$ denote the directional derivative of $\mathbf{n}$ in the direction $\mathbf{u} \in T_{p}\mathcal{M}$. The \emph{second fundamental form} at $p$ is the symmetric bilinear form

\begin{equation}
	\mathcal{I\!I}_{p} :
	T_{p}\mathcal{M} \times T_{p}\mathcal{M}
	\longrightarrow \mathbb{R},
	\qquad
	\mathcal{I\!I}_{p}(\mathbf{u}, \mathbf{v})
	\;=\;
	-\,\bigl\langle
	D_{\mathbf{u}}\,\mathbf{n},\; \mathbf{v}
	\bigr\rangle_{\mathbb{R}^{m}}.
	\label{eq:second_ff_def}
\end{equation}

The negative sign is conventional and ensures that the second fundamental form is positive when the manifold curves toward the normal direction \citep{ONeill2006}. Symmetry of $\mathcal{I\!I}_{p}$ follows from the identity $\langle D_{\mathbf{u}}\,\mathbf{n}, \mathbf{v}\rangle = \langle D_{\mathbf{v}}\,\mathbf{n}, \mathbf{u}\rangle$, which holds for any smooth surface and is a consequence of the symmetry of mixed partial derivatives \citep{do_carmo_differential_2016}. An equivalent and computationally convenient expression is obtained by differentiating the constraint $\langle \mathbf{n}(p), \mathbf{v}\rangle = 0$ along $\mathcal{M}$:

\begin{equation}
	\mathcal{I\!I}_{p}(\mathbf{u}, \mathbf{v})
	\;=\;
	\bigl\langle
	\mathbf{n}(p),\; D_{\mathbf{u}}\,\mathbf{v}
	\bigr\rangle_{\mathbb{R}^{m}},
	\label{eq:second_ff_alt}
\end{equation}

which relates $\mathcal{I\!I}$ to the ambient acceleration of tangent vector fields rather than to the derivative of the normal. Through this expression, one can verify that $\mathcal{I\!I}_{p}(\mathbf{t}, \mathbf{t}) = \kappa_{\mathbf{n}}(\mathbf{t})$, confirming that the second fundamental form encodes the normal curvature in every tangent direction \citep{spivak1999comprehensive}.

\paragraph{Matrix representation in local coordinates.}
Let $\{\partial_{1}, \ldots, \partial_{d}\}$ be the coordinate basis of $T_{p}\mathcal{M}$ induced by a local parametrisation $\boldsymbol{\phi}$ as in~\eqref{eq:coord_basis}. The second fundamental form is completely determined by the \emph{curvature matrix} $L = (l_{ij}) \in \mathbb{R}^{d \times d}$ with entries

\begin{equation}
	l_{ij}
	\;=\;
	\mathcal{I\!I}_{p}(\partial_{i},\, \partial_{j})
	\;=\;
	-\left\langle
	\frac{\partial \mathbf{n}}{\partial u^{i}},\;
	\partial_{j}
	\right\rangle_{\mathbb{R}^{m}}
	\;=\;
	\left\langle
	\mathbf{n},\;
	\frac{\partial^{2} \boldsymbol{\phi}}{\partial u^{i}
		\partial u^{j}}
	\right\rangle_{\mathbb{R}^{m}},
	\label{eq:curvature_matrix}
\end{equation}

where the second equality uses the smoothness of $\boldsymbol{\phi}$ and differentiation of the orthogonality constraint $\langle \mathbf{n}, \partial_{j}\rangle = 0$ \citep{do_carmo_differential_2016, DGApp}. The symmetry $l_{ij} = l_{ji}$ is again a consequence of the equality of mixed partial derivatives. For an arbitrary unit tangent vector
$\mathbf{v} = \sum_{i} v^{i}\partial_{i}$, the normal curvature in the direction $\mathbf{v}$ is

\begin{equation}
	\kappa_{\mathbf{n}}(\mathbf{v})
	\;=\;
	\frac{\mathcal{I\!I}_{p}(\mathbf{v},\mathbf{v})}
	{g_{p}(\mathbf{v},\mathbf{v})}
	\;=\;
	\frac{\mathbf{v}^{\top} L\, \mathbf{v}}
	{\mathbf{v}^{\top} G\, \mathbf{v}},
	\label{eq:normal_curv_ratio}
\end{equation}

a Rayleigh-type quotient whose extrema over all unit tangent vectors yield the principal curvatures $\kappa_{1} \geq \kappa_{2} \geq \cdots \geq \kappa_{d}$ \citep{ONeill2006}.

\paragraph{Geometric interpretation.}
Equation~\eqref{eq:normal_curv_ratio} reveals that the second fundamental form governs the full spectrum of curvature at each point. Its eigenstructure, determined by the generalised eigenproblem $L\mathbf{v} = \kappa\, G\mathbf{v}$, simultaneously provides the principal curvatures $\kappa_{i}$ (eigenvalues) and the principal curvature directions (eigenvectors).
When all $\kappa_{i} > 0$, the manifold is locally \emph{convex}, bending toward the normal on every side; when signs are mixed, the point is a \emph{saddle}, with the manifold curving toward the normal in some directions and away in others; when all $\kappa_{i} = 0$, the point is locally \emph{flat} and $\mathcal{I\!I}_{p} \equiv 0$ \citep{spivak1999comprehensive, needham2021visual}. These geometric cases translate directly into data-analytic distinctions: convex regions correspond to the interior of dense clusters, flat regions to smooth manifold patches, and saddle or mixed-curvature regions to geometric transitions and boundaries between data classes.

\paragraph{Extension to higher co-dimension.}
For hypersurfaces ($d = m-1$), the unit normal $\mathbf{n}(p)$ is unique up to sign, and the second fundamental form~\eqref{eq:second_ff_def} is well-defined as a scalar-valued bilinear form. For embeddings of higher co-dimension ($d < m-1$), the normal space $N_{p}\mathcal{M}$ has dimension greater than one, and for each unit normal $\boldsymbol{\nu} \in N_{p}\mathcal{M}$ one obtains a separate scalar-valued form

\begin{equation}
	\mathcal{I\!I}^{\boldsymbol{\nu}}_{p}(\mathbf{u}, \mathbf{v})
	\;=\;
	-\,\bigl\langle
	D_{\mathbf{u}}\,\boldsymbol{\nu},\; \mathbf{v}
	\bigr\rangle_{\mathbb{R}^{m}}.
	\label{eq:second_ff_high_codim}
\end{equation}

The full extrinsic curvature is then encoded by the \emph{vector-valued} second fundamental form $\vec{\mathcal{I\!I}}_{p}(\mathbf{u},\mathbf{v}) \in N_{p}\mathcal{M}$, whose projection onto each normal direction recovers~\eqref{eq:second_ff_high_codim} \citep{spivak1999comprehensive, do_carmo_differential_2016}. In the high-dimensional data-analytic setting considered in this paper, where $\mathcal{M}$ is a $d$-dimensional manifold with $d \ll m$, the co-dimension $m - d$ can be large, and the mean curvature is defined as the trace of the shape operator relative to a chosen normal direction, a construction made precise in the following subsection.

\paragraph{Relevance in the proposed method.}
The second fundamental form occupies a central role in the curvature estimator of \citet{mcbp2025} that motivates this paper. Specifically, the feature matrix $H = [W^{(2)} \mid W^{(\times)}]$ introduced in Section~\ref{subsec:setup} is a discrete approximation of the curvature matrix $L$, constructed from the eigenvectors of the local sample covariance $\boldsymbol{\Sigma}$. The element-wise squared columns $W^{(2)}$ encode the diagonal entries
$l_{ii}$ of $L$, the normal curvatures along the principal directions, while the cross-product columns $W^{(\times)}$ encode the off-diagonal entries $l_{ij}$, capturing the coupling between distinct curvature directions. The product $HH^{\top}$ thus approximates the Gram matrix of the curvature tensor, and its contraction with $\boldsymbol{\Sigma}$ recovers the trace of the shape operator $\mathcal{S} = G^{-1}L$, that is, the mean curvature as defined in \citep{mcbp2025, cheng2021weingarten}. The closed-form identity of Theorem~\ref{thm:identity} then shows that this entire computation, which naively requires forming a matrix with $O(m^{2})$ columns, reduces to a single $m \times m$ matrix product through the orthogonality of the eigenvector frame, an algebraic consequence of the fact that, in the PCA basis, the metric $G$ becomes the identity and simultaneously diagonalises $L$, so that the principal curvatures are read off directly from the eigenvalues of $\boldsymbol{\Sigma}$.

\subsection{Shape Operator}
\label{subsec:shape_op}

The first and second fundamental forms introduced in the preceding subsections encode intrinsic and extrinsic curvature information, respectively, but they do so in a coordinate-dependent way: their matrix representations $G$ and $L$ change under reparametrisation, making it difficult to extract coordinate-free geometric invariants
directly from them. The \emph{shape operator}, also called the \emph{Weingarten map}, resolves this issue by amalgamating $G$ and $L$ into a single, self-adjoint linear endomorphism of the tangent space whose eigenvalues and trace are genuine geometric invariants, independent of any choice of coordinates or embedding \citep{do_carmo_differential_2016, ONeill2006, spivak1999comprehensive}.

\paragraph{Definition.}
Let $p \in \mathcal{M}$, and let $g_{p}$ and $\mathcal{I\!I}_{p}$ be the first and second fundamental forms at $p$, respectively. Because $g_{p}$ is non-degenerate (positive definite), for every fixed $\mathbf{u} \in T_{p}\mathcal{M}$ the map $\mathbf{v} \mapsto \mathcal{I\!I}_{p}(\mathbf{u}, \mathbf{v})$ is a linear functional on $T_{p}\mathcal{M}$, and by the Riesz representation theorem there exists a unique vector in $T_{p}\mathcal{M}$, denoted $\mathcal{S}_{p}(\mathbf{u})$, such that

\begin{equation}
	\mathcal{I\!I}_{p}(\mathbf{u}, \mathbf{v})
	\;=\;
	g_{p}\!\bigl(\mathcal{S}_{p}(\mathbf{u}),\, \mathbf{v}\bigr),
	\qquad
	\forall\, \mathbf{v} \in T_{p}\mathcal{M}.
	\label{eq:shape_op_def}
\end{equation}

The map $\mathcal{S}_{p} : T_{p}\mathcal{M} \to T_{p}\mathcal{M}$ defined by~\eqref{eq:shape_op_def} is called the \emph{shape operator} at $p$. It is linear and, because $\mathcal{I\!I}_{p}$ is symmetric, it is self-adjoint with respect to $g_{p}$:

\begin{equation}
	g_{p}\!\bigl(\mathcal{S}_{p}(\mathbf{u}),\, \mathbf{v}\bigr)
	\;=\;
	g_{p}\!\bigl(\mathbf{u},\, \mathcal{S}_{p}(\mathbf{v})\bigr),
	\qquad
	\forall\, \mathbf{u}, \mathbf{v} \in T_{p}\mathcal{M}.
	\label{eq:shape_op_selfadj}
\end{equation}

An equivalent characterisation, more amenable to computation, is obtained by differentiating the unit normal field $\mathbf{n}$:

\begin{equation}
	\mathcal{S}_{p}(\mathbf{u})
	\;=\;
	-\,\bigl(D_{\mathbf{u}}\,\mathbf{n}\bigr)^{\!\top},
	\label{eq:shape_op_dn}
\end{equation}

where $(\,\cdot\,)^{\!\top}$ denotes the orthogonal projection onto $T_{p}\mathcal{M}$.
Equation~\eqref{eq:shape_op_dn} is the \emph{Weingarten equation}: it states that the shape operator measures the rate of change of the unit normal as one moves along the manifold, projected back onto the tangent space. Intuitively, if the normal rotates rapidly as $p$ moves in the direction $\mathbf{u}$, then the manifold bends sharply in that direction, and $\mathcal{S}_{p}(\mathbf{u})$ is correspondingly large \citep{ONeill2006, needham2021visual}.

\paragraph{Matrix representation in local coordinates.}
In the coordinate basis $\{\partial_{1}, \ldots, \partial_{d}\}$ introduced in~\eqref{eq:coord_basis}, the shape operator is represented by the $d \times d$ matrix

\begin{equation}
	\mathcal{S}
	\;=\;
	G^{-1} L,
	\label{eq:shape_matrix}
\end{equation}

where $G = (g_{ij})$ is the Gram matrix of the first fundamental form and $L = (l_{ij})$ is the curvature matrix of the second fundamental form, both defined in Sections~\ref{subsubsec:first_ff} and~\ref{subsubsec:second_ff}. To verify~\eqref{eq:shape_matrix}, note that the defining relation~\eqref{eq:shape_op_def} in coordinates reads

\begin{equation}
	l_{ij}
	\;=\;
	\mathcal{I\!I}_{p}(\partial_{i}, \partial_{j})
	\;=\;
	g_{p}\!\bigl(\mathcal{S}(\partial_{i}),\, \partial_{j}\bigr)
	\;=\;
	\sum_{k=1}^{d} \mathcal{S}_{ki}\, g_{kj},
	\label{eq:shape_verify}
\end{equation}

which in matrix form is $L = G\mathcal{S}$, hence $\mathcal{S} = G^{-1}L$ \citep{do_carmo_differential_2016, DGApp}. Note that $\mathcal{S}$ is symmetric in the $g$-inner product but not necessarily symmetric as a plain matrix unless the coordinate basis is $g$-orthonormal.

\paragraph{Spectral decomposition and principal curvatures.}
Since $\mathcal{S}_{p}$ is self-adjoint with respect to the inner product $g_{p}$, the spectral theorem guarantees that it admits a complete set of real eigenvalues and $g_{p}$-orthogonal eigenvectors. The eigenvalue problem

\begin{equation}
	\mathcal{S}_{p}(\mathbf{e}_{i})
	\;=\;
	\kappa_{i}\,\mathbf{e}_{i},
	\qquad i = 1, \ldots, d,
	\label{eq:shape_eigen}
\end{equation}

is equivalent in coordinates to the generalised eigenproblem

\begin{equation}
	L\,\mathbf{v}_{i}
	\;=\;
	\kappa_{i}\, G\,\mathbf{v}_{i},
	\label{eq:gen_eigen}
\end{equation}

whose solutions $\kappa_{1} \geq \kappa_{2} \geq \cdots \geq \kappa_{d}$ are the \emph{principal curvatures} of $\mathcal{M}$ at $p$, and the corresponding unit eigenvectors
$\{\mathbf{e}_{1}, \ldots, \mathbf{e}_{d}\}$ are the \emph{principal curvature directions} \citep{ONeill2006, spivak1999comprehensive}. The principal curvature directions are $g_{p}$-orthonormal: $g_{p}(\mathbf{e}_{i}, \mathbf{e}_{j}) = \delta_{ij}$, so they provide a canonical, coordinate-free frame for $T_{p}\mathcal{M}$ that simultaneously diagonalises both fundamental forms:

\begin{equation}
	\bigl[g_{p}(\mathbf{e}_{i},\mathbf{e}_{j})\bigr]
	\;=\;
	I_{d},
	\qquad
	\bigl[\mathcal{I\!I}_{p}(\mathbf{e}_{i},\mathbf{e}_{j})\bigr]
	\;=\;
	\operatorname{diag}(\kappa_{1}, \ldots, \kappa_{d}).
	\label{eq:diag_forms}
\end{equation}

\paragraph{Geometric invariants.}
The two most important scalar invariants of $\mathcal{S}_{p}$ are its trace and determinant.
The \emph{mean curvature} is proportional to the trace:

\begin{equation}
	H(p)
	\;=\;
	\frac{1}{d}\operatorname{tr}\!\bigl(\mathcal{S}_{p}\bigr)
	\;=\;
	\frac{1}{d}\sum_{i=1}^{d}\kappa_{i}
	\;=\;
	\frac{1}{d}\operatorname{tr}(G^{-1}L),
	\label{eq:mean_curv_trace}
\end{equation}

and the \emph{Gaussian curvature} is the determinant:

\begin{equation}
	K(p)
	\;=\;
	\det\!\bigl(\mathcal{S}_{p}\bigr)
	\;=\;
	\prod_{i=1}^{d}\kappa_{i}
	\;=\;
	\frac{\det L}{\det G}.
	\label{eq:gauss_curv}
\end{equation}

Both $H(p)$ and $K(p)$ are invariant under rigid motions of $\mathbb{R}^{m}$ and under reparametrisation of $\mathcal{M}$. By the celebrated \emph{Theorema Egregium} of Gauss, the Gaussian curvature $K$ is in fact an intrinsic invariant, it can be computed from $g$ alone, without reference to the embedding, while the mean curvature $H$ is genuinely extrinsic and depends on how $\mathcal{M}$ sits in $\mathbb{R}^{m}$ \citep{do_carmo_differential_2016, spivak1999comprehensive}. It is precisely this extrinsic character of $H$ that makes it
sensitive to the bending of the data manifold relative to the ambient feature space, and hence a powerful indicator of geometric boundary structure in high-dimensional datasets.

\paragraph{Simplification in the PCA eigenvector basis.}
In the data-driven context of this paper, the tangent frame is estimated via the leading eigenvectors of the local sample covariance matrix $\boldsymbol{\Sigma} = W\operatorname{diag}(\mathbf{v})W^{\top}$, as described in Section~\ref{subsec:setup}.
Since $W$ is orthogonal and its columns estimate the principal curvature directions, working in this eigenvector basis has the important consequence of simultaneously setting $G = I_{m}$, because PCA produces an orthonormal frame, so that the shape operator~\eqref{eq:shape_matrix} simplifies to

\begin{equation}
	\mathcal{S}
	\;=\;
	G^{-1}L
	\;\xrightarrow{\;G\,=\,I_{m}\;}
	\;L,
	\label{eq:shape_simplified}
\end{equation}

and its trace, the mean curvature, reduces to

\begin{equation}
	H(p)
	\;=\;
	\frac{1}{d}\operatorname{tr}(L)
	\;=\;
	\frac{1}{d}\operatorname{tr}(\mathcal{S}).
	\label{eq:mean_curv_pca}
\end{equation}

This simplification is not merely notational: it is the geometric reason why the discrete estimator $\kappa_{i} \propto |\operatorname{tr}(-HH^{\top}\boldsymbol{\Sigma})|$ of \citet{mcbp2025} takes the specific form it does, and why the orthogonality of $W$, formally encoded as $W^{\top}W = I_{m}$, is the algebraic key that enables the closed-form reduction from $O(m^{4})$ to $O(m^{2})$ established in Theorem~\ref{thm:identity}.

\paragraph{Relevance in the proposed method.}
From the perspective of the algorithm developed in this paper, the shape operator plays three distinct roles. First, it provides the theoretical target that the discrete estimator $HH^{\top}$ approximates: the matrix $H$, whose columns are the element-wise squares and cross-products of the eigenvectors of $\boldsymbol{\Sigma}$, encodes a discretisation of $\mathcal{S}$ projected onto the local neighborhood. Second, the self-adjointness of $\mathcal{S}$, and its coordinate-free definition via~\eqref{eq:shape_op_def}, guarantees
that the mean curvature $H(p) = d^{-1}\operatorname{tr}(\mathcal{S})$ is a well-defined geometric quantity independent of the orientation of the eigenvector frame, so that the estimator is consistent across different points of the dataset even when the eigenvectors are not globally aligned. Third, the eigenvalues $\kappa_{i}$ of $\mathcal{S}$ provide a fine-grained description of local geometry that, beyond the scalar mean curvature used in \citet{mcbp2025}, could in principle be used to construct richer curvature-based features, such as the full curvature spectrum or the Gaussian curvature $K = \prod_{i}\kappa_{i}$, for downstream machine learning tasks such as anomaly detection, semi-supervised classification, or geometric graph construction \citep{cheng2021weingarten, Curvgad}.

Figure~\ref{fig:shape} provides a geometric illustration of the shape operator, showing how the rate of change of the unit normal field $\mathbf{n}$ encodes the local bending of a curved surface. At each point $p \in \mathcal{M}$, the shape operator $\mathcal{S}_{p}$ acts on a tangent vector $\mathbf{v} \in T_{p}\mathcal{M}$ and returns, via the Weingarten
equation~\eqref{eq:shape_op_dn}, the component of $D_{\mathbf{v}}\,\mathbf{n}$ projected back onto the tangent space, that is, the infinitesimal rotation of the normal as one moves from
$p$ in the direction $\mathbf{v}$. Crucially, this rate of change is \emph{anisotropic}: the normal rotates at different speeds depending on the chosen tangent direction, so $\mathcal{S}_{p}$ is not a scalar but a self-adjoint linear map whose eigenvectors identify the directions of extremal bending, the principal curvature directions, and whose eigenvalues $\kappa_{i}$ measure the corresponding rates of rotation. A tangent direction aligned with a principal curvature direction $\mathbf{e}_{i}$ yields $\mathcal{S}_{p}(\mathbf{e}_{i}) =
\kappa_{i}\,\mathbf{e}_{i}$, so the normal rotates purely in that direction at rate $\kappa_{i}$; an arbitrary tangent direction combines these extremal responses through the spectral decomposition~\eqref{eq:diag_forms}.

\begin{figure}
	\centering
	\includegraphics[scale=0.6]{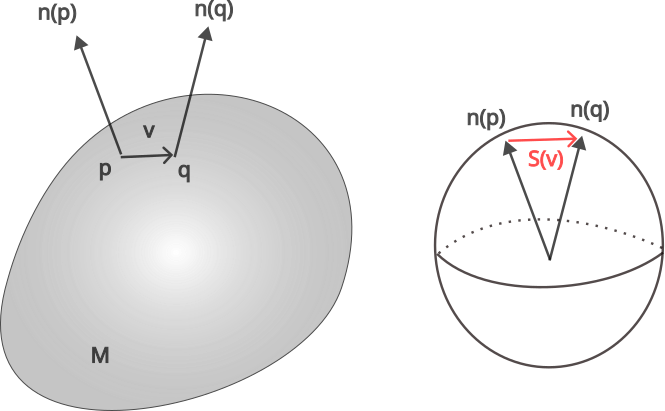}
	\caption{Geometric illustration of the shape operator $\mathcal{S}_{p}$ on a curved surface $\mathcal{M}$. At the base point $p$, the unit normal $\mathbf{n}(p)$ is orthogonal to the tangent plane $T_{p}\mathcal{M}$.	As one moves from $p$ in the tangent direction $\mathbf{v} \in T_{p}\mathcal{M}$, the normal field rotates at a rate that depends on $\mathbf{v}$: the shape operator returns the tangential component of this
    rate of change, $\mathcal{S}_{p}(\mathbf{v}) = -(D_{\mathbf{v}}\,\mathbf{n})^{\!\top}$.
	Directions along which $\mathbf{n}$ rotates fastest and	slowest are the principal curvature directions $\mathbf{e}_{1}$ and $\mathbf{e}_{2}$, with principal curvatures $\kappa_{1} \geq	\kappa_{2}$; their arithmetic mean $H = (\kappa_{1} + \kappa_{2})/2$ is the mean curvature at $p$.}
	\label{fig:shape}
\end{figure}

\section{Local Shape Operator Estimation}
\label{sec:method}

Let $\mathcal{X} = \{\mathbf{x}_{1}, \mathbf{x}_{2}, \ldots, \mathbf{x}_{n}\} \subset \mathbb{R}^{D}$ be a finite dataset sampled from an unknown smooth $d$-dimensional manifold
$\mathcal{M} \hookrightarrow \mathbb{R}^{D}$, with $d \ll D$. Our goal is to estimate, at each point $\mathbf{x}_{i} \in \mathcal{X}$, a discrete approximation of the shape operator $\mathcal{S}_{i}$ and, from it, the local mean curvature $\mathcal{K}_{i}$ — using only
the geometry of a small $k$-nearest-neighbour patch around $\mathbf{x}_{i}$. The procedure consists of five steps, described below.

\paragraph{Step 1: Local neighborhood Construction.}

For each point $\mathbf{x}_{i}$, we identify its $k$-nearest neighbours under the Euclidean metric:

\begin{equation}
	\mathcal{N}_{i}
	\;=\;
	\mathrm{kNN}(\mathbf{x}_{i},\, k)
	\;=\;
	\{\mathbf{x}_{i_{1}}, \ldots, \mathbf{x}_{i_{k}}\}
	\;\subset\; \mathcal{X}.
	\label{eq:knn}
\end{equation}

Under the manifold hypothesis, the set $\mathcal{N}_{i}$ provides a discrete approximation of a local coordinate chart around $\mathbf{x}_{i}$: for $k$ sufficiently large relative to the local reach of $\mathcal{M}$, the convex hull of $\mathcal{N}_{i}$ is contained in a thin tubular neighborhood of $\mathcal{M}$, and the geodesic distances within $\mathcal{N}_{i}$ are well approximated by Euclidean distances \citep{fefferman2016testing, tenenbaum2000global}. The parameter $k$ governs a fundamental bias–variance trade-off: small
$k$ yields sharper local estimates but increases sensitivity to noise, while large $k$ improves statistical stability at the cost of smoothing out fine geometric features.

\paragraph{Step 2: Local Covariance and Metric Approximation.}

Let $\bar{\mathbf{x}}_{i} = k^{-1}\sum_{\mathbf{x}_{j} \in \mathcal{N}_{i}} \mathbf{x}_{j}$ be the local centroid. We compute the sample covariance matrix of the neighborhood:

\begin{equation}
	\mathbf{C}_{i}
	\;=\;
	\frac{1}{k-1}
	\sum_{\mathbf{x}_{j} \in \mathcal{N}_{i}}
	\bigl(\mathbf{x}_{j} - \bar{\mathbf{x}}_{i}\bigr)
	\bigl(\mathbf{x}_{j} - \bar{\mathbf{x}}_{i}\bigr)^{\!\top}
	\;\in\; \mathbb{R}^{D \times D}.
	\label{eq:cov}
\end{equation}

The matrix $\mathbf{C}_{i}$ captures the anisotropic dispersion of the data within $\mathcal{N}_{i}$ and provides a first-order approximation of the local geometry of $\mathcal{M}$ near $\mathbf{x}_{i}$: its leading eigenvectors span an estimate of the tangent space $T_{\mathbf{x}_{i}}\mathcal{M}$, while its eigenvalues reflect the local extent of the manifold along each direction \citep{Singer2012}. Following a Riemannian interpretation in which the metric tensor is adapted to the local data distribution, we associate the inverse
covariance with the local metric:

\begin{equation}
	\mathbf{g}_{i}
	\;\approx\;
	\mathbf{C}_{i}^{-1},
	\label{eq:metric_approx}
\end{equation}

which corresponds to a Mahalanobis-type metric that contracts distances along high-variance directions and expands them along low-variance ones, thereby de-emphasising ambient dimensions irrelevant to the local manifold structure \citep{do_carmo_differential_2016}.
When $\mathbf{C}_{i}$ is rank-deficient, as is generically the case when $k \leq D$, the inverse in~\eqref{eq:metric_approx} is interpreted as a Moore–Penrose pseudoinverse restricted to the column space of $\mathbf{C}_{i}$.

\vspace{1cm}

\paragraph{Step 3: Local Frame and Second-Order Structure.}

Let

\begin{equation}
	\mathbf{C}_{i}
	\;=\;
	\mathbf{W}_{i}\,\boldsymbol{\Lambda}_{i}\,\mathbf{W}_{i}^{\!\top}
	\label{eq:eigen}
\end{equation}

be the eigendecomposition of $\mathbf{C}_{i}$, where $\boldsymbol{\Lambda}_{i} = \operatorname{diag}(\lambda_{1}^{(i)}, \ldots, \lambda_{D}^{(i)})$ with $\lambda_{1}^{(i)} \geq \cdots \geq \lambda_{D}^{(i)} \geq 0$, and $\mathbf{W}_{i} = [\mathbf{w}_{1}^{(i)} \mid \cdots \mid \mathbf{w}_{D}^{(i)}]$ is the orthogonal matrix of corresponding eigenvectors.
The leading $d$ columns of $\mathbf{W}_{i}$, those associated with the $d$ largest eigenvalues, span the estimated tangent space $\widehat{T}_{\mathbf{x}_{i}}\mathcal{M}$, while the remaining $D - d$ columns span the estimated normal space $\widehat{N}_{\mathbf{x}_{i}}\mathcal{M}$, consistent with the decomposition $\mathbb{R}^{D} = T_{\mathbf{x}_{i}}\mathcal{M} \oplus N_{\mathbf{x}_{i}}\mathcal{M}$ established in
Section~\ref{subsec:tangent_spaces}.

To capture second-order geometric information, that is, the local curvature of $\mathcal{M}$, we construct a feature matrix $\mathbf{H}_{i} \in \mathbb{R}^{D \times p}$, with
$p = D + D(D-1)/2$, whose columns are:

\begin{itemize}
	\item the \emph{quadratic terms}
	$\bigl(\mathbf{w}_{j}^{(i)}\bigr)^{\circ 2}$,
	for $j = 1, \ldots, D$, encoding the self-interaction of each eigenvector direction; and
	\item the \emph{cross terms} $\mathbf{w}_{j}^{(i)} \circ \mathbf{w}_{\ell}^{(i)}$,
	for $1 \leq j < \ell \leq D$, encoding the interaction between pairs of eigenvector directions,
\end{itemize}

where $\circ$ denotes the Hadamard (element-wise) product. The matrix $\mathbf{H}_{i}$ constitutes a local quadratic expansion of the eigenvector frame, and its columns span a feature space that is sensitive to deviations of $\mathcal{M}$ from its tangent hyperplane.
The second fundamental form is then approximated by the Gram matrix of this feature space:

\begin{equation}
	\widehat{\mathcal{I\!I}}_{i}
	\;=\;
	\mathbf{H}_{i}\mathbf{H}_{i}^{\!\top}
	\;\in\; \mathbb{R}^{D \times D}.
	\label{eq:second_ff_approx}
\end{equation}

This construction provides a data-driven estimate of the local curvature tensor by measuring, in a least-squares sense, the degree to which the neighborhood $\mathcal{N}_{i}$ departs from the linear (flat) approximation provided by $\widehat{T}_{\mathbf{x}_{i}} \mathcal{M}$ \citep{mcbp2025, cheng2021weingarten}.

\paragraph{Step 4: Shape Operator Estimation.}

Following the coordinate representation established in Section~\ref{subsec:shape_op}, the shape operator is the linear map $\mathcal{S}_{i} : T_{\mathbf{x}_{i}}\mathcal{M} \to
T_{\mathbf{x}_{i}}\mathcal{M}$ given by $\mathcal{S} = G^{-1}L$ in local coordinates — that is, the composition of the curvature matrix with the inverse metric. Substituting the approximations~\eqref{eq:metric_approx} and~\eqref{eq:second_ff_approx}, we obtain the discrete estimator:

\begin{equation}
	\widehat{\mathcal{S}}_{i}
	\;=\;
	-\,\widehat{\mathcal{I\!I}}_{i}\;\mathbf{g}_{i}^{-1}
	\;=\;
	-\,\mathbf{H}_{i}\mathbf{H}_{i}^{\!\top}\,\mathbf{C}_{i},
	\label{eq:shape_est}
\end{equation}

where the negative sign follows the convention established in~\eqref{eq:second_ff_def}, ensuring that positive eigenvalues correspond to convex bending toward the normal. The estimator~\eqref{eq:shape_est} is computable entirely from the eigendecomposition of $\mathbf{C}_{i}$ and requires no explicit knowledge of the embedding map or the unit normal field.

\paragraph{Step 5: Mean Curvature Estimation.}

The shape operator $\widehat{\mathcal{S}}_{i}$ encodes the full local curvature structure of $\mathcal{M}$ near $\mathbf{x}_{i}$: its eigenvalues $\widehat{\kappa}_{1}^{(i)} \geq \cdots \geq \widehat{\kappa}_{d}^{(i)}$ approximate the principal curvatures, and its eigenvectors approximate the principal curvature directions. The trace of $\widehat{\mathcal{S}}_{i}$, which is invariant under change of basis and equals the sum of principal curvatures, yields the mean curvature estimator:

\begin{equation}
	\mathcal{K}_{i}
	\;=\;
	\bigl|\operatorname{tr}\!\bigl(\widehat{\mathcal{S}}_{i}\bigr)\bigr|
	\;=\;
	\bigl|\operatorname{tr}\!\bigl(
	-\mathbf{H}_{i}\mathbf{H}_{i}^{\!\top}\mathbf{C}_{i}
	\bigr)\bigr|.
	\label{eq:mean_curv_est}
\end{equation}

The absolute value accommodates the sign ambiguity of the unit normal in high co-dimension, ensuring that $\mathcal{K}_{i} \geq 0$ regardless of the orientation of  $\widehat{N}_{\mathbf{x}_{i}}\mathcal{M}$. This scalar quantity provides a coordinate-free measure of local geometric complexity: high values of $\mathcal{K}_{i}$ identify regions where the manifold bends sharply and local linear approximations are least reliable, while low values correspond to near-flat regions where the tangent space provides an accurate
description of the local structure. As established in Theorem~\ref{thm:identity}, the naive evaluation of~\eqref{eq:mean_curv_est}, which requires forming the $D \times p$ matrix $\mathbf{H}_{i}$ with $p = O(D^{2})$ columns and computing the product $\mathbf{H}_{i}\mathbf{H}_{i}^{\!\top}$ at cost $O(D^{4})$, can be replaced by the closed-form expression

\begin{equation}
	\mathcal{K}_{i}
	\;=\;
	\left|
	\frac{1}{2}\,\boldsymbol{\lambda}_{i}^{\!\top}
	\bigl(C_{i}^{\odot 2}\,\mathbf{1}_{D}\bigr)
	+
	\frac{1}{2}\sum_{s=1}^{D}\lambda_{s}^{(i)}
	\right|,
	\qquad
	C_{i} = \mathbf{W}_{i}^{\!\top}\mathbf{W}_{i}^{(2)},
	\label{eq:mean_curv_fast}
\end{equation}

at a total cost of $O(D^{2})$ after the eigendecomposition, where $\mathbf{W}_{i}^{(2)}$ denotes the element-wise square of $\mathbf{W}_{i}$, $C_{i}^{\odot 2}$ the element-wise square of $C_{i}$, and $\boldsymbol{\lambda}_{i} = (\lambda_{1}^{(i)}, \ldots,
\lambda_{D}^{(i)})^{\!\top}$ the vector of eigenvalues of $\mathbf{C}_{i}$.

\section{The Proposed Algebraic Identity for Efficient Mean Curvature Computation}
\label{sec:identity}

In this section we derive the central theoretical contribution of this paper: an exact closed-form expression for the mean curvature estimator that replaces an $O(m^{4})$ tensor contraction with an $O(m^{2})$ matrix product.  The derivation proceeds in four self-contained steps. We begin by restating the original formulation, then expand the key
quadratic form algebraically, exploit the orthogonality of the eigenvector matrix to collapse the expression, and finally reformulate the result as a standard matrix operation whose computational cost is dominated by a single matrix multiplication of size $m\times m$.

\subsection{Problem Setup and Original Formulation}
\label{subsec:setup}

Let $\mathcal{X} = \{\mathbf{x}_{1},\dots,\mathbf{x}_{n}\} \subset \mathbb{R}^{m}$ be a dataset of $n$ points in an $m$-dimensional ambient space.  For each point $\mathbf{x}_{i}$, let $\mathcal{N}(i) = \{\mathbf{x}_{i_{1}},\dots,\mathbf{x}_{i_{k}}\}$ denote its $k$ nearest neighbours under the Euclidean metric, and let

\begin{equation}
	\boldsymbol{\Sigma}_{i}
	\;=\;
	\frac{1}{k-1}
	\sum_{j=1}^{k}
	\bigl(\mathbf{x}_{i_{j}} - \bar{\mathbf{x}}_{i}\bigr)
	\bigl(\mathbf{x}_{i_{j}} - \bar{\mathbf{x}}_{i}\bigr)^{\!\top}
	\;\in\; \mathbb{R}^{m \times m}
	\label{eq:cov}
\end{equation}

be the sample covariance matrix of the local neighborhood, where $\bar{\mathbf{x}}_{i} = k^{-1}\sum_{j=1}^{k}\mathbf{x}_{i_{j}}$ is the local centroid.  Because $\boldsymbol{\Sigma}_{i}$ is real symmetric and positive semi-definite, it admits the spectral
decomposition

\begin{equation}
	\boldsymbol{\Sigma}_{i}
	\;=\;
	W \operatorname{diag}(v_{1},\dots,v_{m}) W^{\!\top},
	\label{eq:spectral}
\end{equation}

where $W = [\mathbf{w}_{1}\mid\cdots\mid\mathbf{w}_{m}] \in \mathbb{R}^{m \times m}$ is the orthogonal matrix whose $l$-th column $\mathbf{w}_{l}\in\mathbb{R}^{m}$ is the eigenvector associated with eigenvalue $v_{l}$, ordered so that $v_{1}\geq v_{2}\geq\cdots\geq
v_{m}\geq 0$.  The orthogonality condition reads

\begin{equation}
	W^{\!\top} W \;=\; W W^{\!\top} \;=\; I_{m},
	\label{eq:orth}
\end{equation}

where $I_{m}$ is the $m\times m$ identity matrix. Following the discrete shape-operator approximation introduced for the MCBP algorithm \citep{mcbp2025}, the pointwise mean curvature estimator at $\mathbf{x}_{i}$ is defined as

\begin{equation}
	\kappa_{i}
	\;=\;
	\bigl|\operatorname{tr}\!\bigl(-H H^{\!\top} \boldsymbol{\Sigma}_{i}\bigr)\bigr|,
	\label{eq:curvature}
\end{equation}

where $H \in \mathbb{R}^{m \times (m + n_{c})}$, with $n_{c} = \binom{m}{2} = m(m-1)/2$, is the feature matrix constructed by concatenating the element-wise squares and pairwise products of the eigenvectors:

\begin{equation}
	H
	\;=\;
	\Bigl[
	\underbrace{W^{(2)}}_{\text{squared columns}}
	\;\Big|\;
	\underbrace{W^{(\times)}}_{\text{cross-product columns}}
	\Bigr].
	\label{eq:H}
\end{equation}

Specifically, the $l$-th column of $W^{(2)}\in\mathbb{R}^{m\times m}$ is $\mathbf{w}_{l}\odot\mathbf{w}_{l}$ (the Hadamard square of the $l$-th eigenvector), while the $(j,l)$-th column of $W^{(\times)}\in\mathbb{R}^{m\times n_{c}}$, for all pairs $1\leq j < l \leq m$, is $\mathbf{w}_{j}\odot\mathbf{w}_{l}$.  Here, $\odot$ denotes the Hadamard (element-wise) product.

Evaluating \eqref{eq:curvature} directly requires forming $H H^{\!\top} \in \mathbb{R}^{m\times m}$ from a matrix with $m + n_{c} \approx m^{2}/2$ columns, followed by a matrix product with $\boldsymbol{\Sigma}_{i}$ and a trace operation.  The dominant cost is
the formation of $H H^{\!\top}$, which involves summing $m + n_{c}$ rank-one outer products of $m$-dimensional vectors, yielding a total complexity of $O(m^{2} \cdot (m + n_{c})) = O(m^{4})$ per point.  For large $m$ this renders direct evaluation intractable.

\subsection{Expansion of \texorpdfstring{$HH^\top$}{HHt}}
\label{subsec:expand}

We begin by decomposing the quadratic form $H H^{\!\top}$ into its two structural components.  From~\eqref{eq:H} we have

\begin{equation}
	H H^{\!\top}
	\;=\;
	W^{(2)}\bigl(W^{(2)}\bigr)^{\!\top}
	+
	W^{(\times)}\bigl(W^{(\times)}\bigr)^{\!\top}.
	\label{eq:HHt_decomp}
\end{equation}

\paragraph{The squared-columns term.}
The $(i,j)$-th entry of the first summand is

\begin{equation}
	\bigl[W^{(2)}\bigl(W^{(2)}\bigr)^{\!\top}\bigr]_{ij}
	\;=\;
	\sum_{l=1}^{m} W_{il}^{2}\,W_{jl}^{2},
	\label{eq:sq_term}
\end{equation}

which is simply the inner product between row $i$ and row $j$ of $W^{(2)}$.

\paragraph{The cross-product term.}
The $(i,j)$-th entry of the second summand sums over all index pairs $j' < l$:

\begin{equation}
	\bigl[W^{(\times)}\bigl(W^{(\times)}\bigr)^{\!\top}\bigr]_{ij}
	\;=\;
	\sum_{\substack{j',l=1\\j'<l}}^{m}
	W_{ij'} W_{il}\cdot W_{jj'} W_{jl}.
	\label{eq:cross_raw}
\end{equation}

To simplify~\eqref{eq:cross_raw} we use the algebraic identity $2\sum_{j'<l} a_{j'} a_{l} = \bigl(\sum_{l} a_{l}\bigr)^{2} -\sum_{l} a_{l}^{2}$ applied to the products $a_{l} =
W_{il}W_{jl}$:

\begin{align}
	\bigl[W^{(\times)}\bigl(W^{(\times)}\bigr)^{\!\top}\bigr]_{ij}
	&=
	\frac{1}{2}
	\left[
	\Bigl(\sum_{l=1}^{m} W_{il}\,W_{jl}\Bigr)^{\!2}
	-
	\sum_{l=1}^{m} W_{il}^{2}\,W_{jl}^{2}
	\right] \notag \\[4pt]
	&=
	\frac{1}{2}
	\Bigl[\bigl(WW^{\!\top}\bigr)_{ij}^{2}
	-
	\bigl[W^{(2)}\bigl(W^{(2)}\bigr)^{\!\top}\bigr]_{ij}\Bigr].
	\label{eq:cross_simplified}
\end{align}

Substituting \eqref{eq:sq_term} and \eqref{eq:cross_simplified} into \eqref{eq:HHt_decomp} gives

\begin{align}
	\bigl[H H^{\!\top}\bigr]_{ij}
	&=
	\sum_{l=1}^{m} W_{il}^{2}W_{jl}^{2}
	+
	\frac{1}{2}
	\Bigl[
	\bigl(WW^{\!\top}\bigr)_{ij}^{2}
	- \sum_{l=1}^{m} W_{il}^{2}W_{jl}^{2}
	\Bigr] \notag \\[4pt]
	&=
	\frac{1}{2}\sum_{l=1}^{m} W_{il}^{2}W_{jl}^{2}
	+
	\frac{1}{2}\bigl(WW^{\!\top}\bigr)_{ij}^{2}.
	\label{eq:HHt_entry}
\end{align}

\subsection{Exploiting Eigenvector Orthogonality}
\label{subsec:orthogonality}

The crucial simplification arises from the orthogonality of $W$.  From~\eqref{eq:orth} we have

\begin{equation}
	\bigl(WW^{\!\top}\bigr)_{ij}
	\;=\;
	\sum_{l=1}^{m} W_{il}\,W_{jl}
	\;=\;
	\delta_{ij},
	\label{eq:delta}
\end{equation}

where $\delta_{ij}$ is the Kronecker delta.  Therefore $\bigl(WW^{\!\top}\bigr)_{ij}^{2} = \delta_{ij}^{2} = \delta_{ij}$, and~\eqref{eq:HHt_entry} reduces to the remarkably simple expression

\begin{equation}
	\boxed{
		H H^{\!\top}
		\;=\;
		\frac{1}{2}\,W^{(2)}\bigl(W^{(2)}\bigr)^{\!\top}
		+
		\frac{1}{2}\,I_{m}.
	}
	\label{eq:HHt_final}
\end{equation}

Equation~\eqref{eq:HHt_final} is the key structural result: the $m\times m$ matrix $H H^{\!\top}$, which implicitly encodes $m + n_{c} \approx m^{2}/2$ feature vectors, equals a linear combination of the identity and the Gram matrix of the element-wise-squared eigenvectors, both of size $m\times m$ and independent of $n_{c}$.

\subsection{Closed-Form Expression for the Curvature Estimator}
\label{subsec:closed_form}

Substituting~\eqref{eq:HHt_final} into~\eqref{eq:curvature} and using linearity of the trace:

\begin{equation}
	\operatorname{tr}\!\bigl(H H^{\!\top} \boldsymbol{\Sigma}_{i}\bigr)
	\;=\;
	\frac{1}{2}\,
	\operatorname{tr}\!\Bigl(W^{(2)}\bigl(W^{(2)}\bigr)^{\!\top}
	\boldsymbol{\Sigma}_{i}\Bigr)
	+
	\frac{1}{2}\,\operatorname{tr}(\boldsymbol{\Sigma}_{i}).
	\label{eq:trace_split}
\end{equation}

The second term is immediate from the spectral decomposition~\eqref{eq:spectral}:

\begin{equation}
	\operatorname{tr}(\boldsymbol{\Sigma}_{i})
	\;=\;
	\sum_{s=1}^{m} v_{s}.
	\label{eq:trace_cov}
\end{equation}

For the first term we apply the cyclic property of the trace, $\operatorname{tr}(ABC) = \operatorname{tr}(CAB)$, to obtain

\begin{equation}
	\operatorname{tr}\!\Bigl(W^{(2)}\bigl(W^{(2)}\bigr)^{\!\top}
	\boldsymbol{\Sigma}_{i}\Bigr)
	\;=\;
	\operatorname{tr}\!\Bigl(\bigl(W^{(2)}\bigr)^{\!\top}
	\boldsymbol{\Sigma}_{i}\,W^{(2)}\Bigr).
	\label{eq:cyclic}
\end{equation}

Inserting the spectral decomposition $\boldsymbol{\Sigma}_{i} = W \operatorname{diag}(\mathbf{v}) W^{\!\top}$ into~\eqref{eq:cyclic} yields

\begin{equation}
	\operatorname{tr}\!\Bigl(\bigl(W^{(2)}\bigr)^{\!\top}
	W\operatorname{diag}(\mathbf{v})W^{\!\top}W^{(2)}\Bigr).
	\label{eq:insert_spectral}
\end{equation}

We now define the $m \times m$ matrix

\begin{equation}
	C \;=\; W^{\!\top} W^{(2)},
	\qquad
	C_{sl} \;=\; \sum_{i=1}^{m} W_{is}\,W_{il}^{2},
	\label{eq:C_def}
\end{equation}

so that~\eqref{eq:insert_spectral} becomes

\begin{equation}
	\operatorname{tr}\!\bigl(C^{\!\top}\operatorname{diag}(\mathbf{v})\,C\bigr)
	\;=\;
	\sum_{s=1}^{m} v_{s}
	\Bigl(\sum_{l=1}^{m} C_{sl}^{2}\Bigr)
	\;=\;
	\mathbf{v}^{\!\top}\bigl(C^{\odot 2}\,\mathbf{1}_{m}\bigr),
	\label{eq:trace_diag}
\end{equation}

where $C^{\odot 2}$ denotes the Hadamard (element-wise) square of $C$ and $\mathbf{1}_{m}\in\mathbb{R}^{m}$ is the all-ones vector, so that $C^{\odot 2}\,\mathbf{1}_{m}$ collects the row sums of $C^{\odot 2}$. Combining~\eqref{eq:trace_split},~\eqref{eq:trace_cov}, and~\eqref{eq:trace_diag}, we arrive at the main result.

\begin{theorem}[Closed-form mean curvature identity]
	\label{thm:identity}
	Let $W\in\mathbb{R}^{m\times m}$ be the orthogonal eigenvector matrix of the local covariance $\boldsymbol{\Sigma}_{i}$, let $\mathbf{v}\in\mathbb{R}^{m}$ be the corresponding eigenvalues (in decreasing order), and let $C = W^{\!\top}W^{(2)}\in\mathbb{R}^{m\times m}$, where $W^{(2)}_{il} = W_{il}^{2}$. Then
	
	\begin{equation}
		\operatorname{tr}\!\bigl(H H^{\!\top} \boldsymbol{\Sigma}_{i}\bigr)
		\;=\;
		\frac{1}{2}\,\mathbf{v}^{\!\top}
		\bigl(C^{\odot 2}\,\mathbf{1}_{m}\bigr)
		+
		\frac{1}{2}\,\mathbf{1}_{m}^{\!\top}\mathbf{v},
		\label{eq:main_identity}
	\end{equation}
	
	and consequently the mean curvature estimator \eqref{eq:curvature} reads
	
	\begin{equation}
		\kappa_{i}
		\;=\;
		\Biggl|
		\frac{1}{2}\,\mathbf{v}^{\!\top}
		\bigl(C^{\odot 2}\,\mathbf{1}_{m}\bigr)
		+
		\frac{1}{2}\sum_{s=1}^{m}v_{s}
		\Biggr|.
		\label{eq:kappa_final}
	\end{equation}
\end{theorem}

\begin{proof}
	The proof follows directly from equations~\eqref{eq:HHt_final} through~\eqref{eq:trace_diag}, which are derived in Sections~\ref{subsec:expand}–\ref{subsec:closed_form}.	The only non-trivial step is~\eqref{eq:HHt_final}, which relies solely on the orthogonality identity $WW^{\!\top}=I_{m}$ established	in~\eqref{eq:orth}.
\end{proof}

\subsection{Complexity Analysis: From \texorpdfstring{$O(m^4)$}{O(m4)}
	to \texorpdfstring{$O(m^2)$}{O(m2)}}
\label{subsec:complexity}

We now give a precise account of the computational savings afforded by Theorem~\ref{thm:identity}.

\paragraph{Original formulation.}
Direct evaluation of~\eqref{eq:curvature} requires the following steps.

\begin{enumerate}
	
	\item \textbf{Forming $H$.}  The matrix $H \in \mathbb{R}^{m \times (m + n_{c})}$ has
	$n_{c} = m(m-1)/2$ cross-product columns, so its construction requires $O(m^{2})$ element-wise vector multiplications of length $m$, totalling $O(m^{3})$ operations.
	
	\item \textbf{Computing $H H^{\!\top}$.}  This is the product of an	$m\times(m+n_{c})$ matrix by its transpose.  With	$m + n_{c} \sim m^{2}/2$ columns, the cost is
	\begin{equation}
		O\!\bigl(m\cdot(m + n_{c})\cdot m\bigr)
		\;=\;
		O\!\bigl(m^{2}\cdot\tfrac{m^{2}}{2}\bigr)
		\;=\;
		O(m^{4}).
		\label{eq:cost_HHt}
	\end{equation}
	This is the dominant term.
	
	\item \textbf{Computing $H H^{\!\top}\boldsymbol{\Sigma}_{i}$ and taking the trace.}  Both operations cost $O(m^{2})$ once $H H^{\!\top}$ is available.
	
\end{enumerate}

The total cost per point is therefore $\boldsymbol{O(m^{4})}$.

\paragraph{Identity-based formulation.}
Evaluation of~\eqref{eq:kappa_final} via Theorem~\ref{thm:identity} decomposes into the following steps.

\begin{enumerate}
	
	\item \textbf{Eigendecomposition of $\boldsymbol{\Sigma}_{i}$.} Computing $W$ and $\mathbf{v}$ via a full symmetric eigensolver	costs $O(m^{3})$.
	
	\item \textbf{Computing $W^{(2)}$.}	Element-wise squaring of the $m\times m$ matrix $W$ costs $O(m^{2})$.
	
	\item \textbf{Computing $C = W^{\!\top}W^{(2)}$.} A standard $m\times m$ matrix multiplication costs $O(m^{3})$. However, this is a \emph{single}, dense matrix product on a modern BLAS routine, which is highly cache-efficient and incurs a	much smaller constant than the equivalent operations implicit in	forming $HH^{\!\top}$.
	
	\item \textbf{Computing $(C^{\odot 2})\mathbf{1}_{m}$.}	Element-wise squaring followed by row summation costs $O(m^{2})$.	
	
	\item \textbf{Dot product $\mathbf{v}^{\!\top}(C^{\odot 2}\mathbf{1}_{m})$ and scalar sum.}	Two inner-product operations, each costing $O(m)$.
	
\end{enumerate}

Steps 3 and 4, comprising the evaluation of the closed-form trace in \eqref{eq:main_identity}, have combined cost $O(m^{3}) + O(m^{2}) = O(m^{3})$.  In practice, however, the eigendecomposition in Step 1 already carries a cost of $O(m^{3})$ and dominates the total clock time once $m$ is large (see Section~\ref{sec:experiments}); Steps 2--5
are negligible by comparison.  The crucial observation is that the explicit construction of $H$ and the $O(m^4)$ computation of $HH^\top$ are \emph{completely eliminated}.

\paragraph{The source of the reduction.}
It is instructive to identify precisely \emph{why} the identity leads to such a dramatic reduction.  The matrix $H$ has $m + n_{c} = O(m^{2})$ columns, making any operation linear in the number of columns automatically $O(m^{2})$ in the leading factor.  The formation of
$HH^{\!\top}$ is quadratic in the number of columns, hence $O(m^{4})$.

The algebraic key is the orthogonality of $W$.  Without it, the expression $\bigl(WW^{\!\top}\bigr)_{ij}^{2}$ in \eqref{eq:HHt_entry} would be a non-trivial $m\times m$ matrix that encodes complex interactions among all $m$ eigenvectors, and no further
simplification would be possible.  Because $WW^{\!\top} = I_{m}$, however, the square $(WW^{\!\top})_{ij}^{2}$ collapses to $\delta_{ij}$ (\emph{cf}.~\eqref{eq:delta}), and the entire contribution of the $n_{c}$ cross-product columns to $HH^{\!\top}$ is captured by a
single diagonal correction, as shown in~\eqref{eq:HHt_final}.

Consequently, the effective rank of the problem is reduced from $m + n_{c}$ to $m$: instead of summing $O(m^{2})$ rank-one outer products of $m$-dimensional vectors, one needs only an $m\times m$ matrix product $W^{\!\top}W^{(2)}$ together with inexpensive element-wise operations.  This is the mechanism by which the complexity drops from $O(m^{4})$ to $O(m^{3})$ for the matrix product and to $O(m^{2})$ for the subsequent trace evaluation, yielding an overall per-point cost of $\boldsymbol{O(m^{3})}$ (dominated by the eigendecomposition) instead of $\boldsymbol{O(m^{4})}$.

Table~\ref{tab:complexity} summarises the computational costs of each step in both formulations.

\begin{table}[t]
	\centering
	\caption{Per-point computational complexity of the original and
		identity-based curvature estimators.  Here $m$ denotes the ambient
		dimension, $k$ the neighborhood size, and $n_{c} = m(m-1)/2$ the
		number of cross-product features.}
	\label{tab:complexity}
	\begin{tabular}{lcc}
		\hline
		\textbf{Step} & \textbf{Original} & \textbf{This work} \\
		\hline
		Covariance matrix $\boldsymbol{\Sigma}_{i}$    & $O(km^{2})$ & $O(km^{2})$ \\
		Eigendecomposition                              & $O(m^{3})$  & $O(m^{3})$  \\
		Forming $H$                                     & $O(m^{3})$  & ---         \\
		Computing $HH^{\!\top}$                         & $O(m^{4})$  & ---         \\
		Computing $W^{(2)}$                             & ---         & $O(m^{2})$  \\
		Computing $C = W^{\!\top}W^{(2)}$               & ---         & $O(m^{3})$  \\
		Trace evaluation                                & $O(m^{2})$  & $O(m^{2})$  \\
		\hline
		\textbf{Total (dominant)}                       & $\boldsymbol{O(m^{4})}$ & $\boldsymbol{O(m^{3})}$ \\
		\hline
	\end{tabular}
\end{table}

\begin{remark}[Numerical equivalence]
	\label{rem:numerical}
	Theorem~\ref{thm:identity} is an exact algebraic equality: no approximation is introduced.  The only source of discrepancy with respect to the original formulation is floating-point rounding, which in practice produces differences of order $10^{-12}$--$10^{-14}$ in double precision, well below any threshold of practical relevance.
\end{remark}

\begin{remark}[Further reduction via truncated SVD]
	\label{rem:fast_mode}
	For datasets in which the ambient dimension $m$ greatly exceeds the neighborhood size $k$, the local covariance matrix $\boldsymbol{\Sigma}_{i}$ has numerical rank at most $p = k-1\ll m$. In this regime the full eigendecomposition is wasteful because $m-p$
	eigenvalues are identically zero.  We exploit this low-rank structure to replace the $O(m^{3})$	eigendecomposition with a truncated SVD of the centered data matrix
	$X_{c}\in\mathbb{R}^{k\times m}$, which costs only $O(k^{2}m)$.	The contribution of the $m-p$ null-space eigenvectors to the curvature	estimator is then handled through an analytical approximation derived under the uniform (Haar) distribution over orthonormal null-space bases, reducing the total per-point cost to $O(k^{2}m + kmp^{2}) = O(k^{2}m)$ for fixed $k\ll m$, a further order-of-magnitude gain over the $O(m^{3})$ exact formulation.
\end{remark}

Algorithm~\ref{alg:emcc} summarises the proposed \textit{Mean Curvature Computation} (MeCuCo) procedure. The method estimates a scalar mean curvature score at each sample point using the closed-form identity of Theorem~\ref{thm:identity}, which replaces the $O(D^{4})$ computation of the original formulation with an $O(D^{2})$ matrix operation after the local eigendecomposition. The algorithm operates in two regimes, selected automatically by the
\texttt{mode} parameter: \textsc{Exact} mode applies the full eigendecomposition of the local covariance matrix via a LAPACK divide-and-conquer solver (driver \texttt{evr}/\texttt{evd}),
yielding exact curvature estimates at cost $O(D^{3})$ per point; \textsc{Fast} mode replaces the eigendecomposition with a truncated singular value decomposition of the $k \times D$ centred neighborhood matrix, exploiting the rank-$(k{-}1)$ structure of the local covariance and an analytical approximation of the null-space contribution, reducing the per-point cost to $O(k^{2}D + kDp^{2})$ with $p = k-1$. In both modes the algorithm is fully data-driven and requires neither an explicit manifold parametrisation nor prior knowledge of the intrinsic dimension $d$, making it directly applicable to high-dimensional datasets. In \textsc{Auto} mode (the default), the algorithm selects \textsc{Exact} when $D < 50$ and \textsc{Fast} otherwise, a threshold calibrated empirically to the crossover point at which the cost of the full eigendecomposition begins to dominate the per-point clock time (see Section~\ref{sec:experiments}).

\begin{algorithm}
	\caption{Efficient Mean Curvature Computation (MeCuCo)}
	\label{alg:emcc}
	\begin{algorithmic}[1]
		\Function{MeCuCo}{$X,\; k,\; \texttt{mode}$}
		\State \textbf{Input:}
		$X \in \mathbb{R}^{n \times D}$ (dataset),
		$k$ (number of neighbours),
		$\texttt{mode} \in \{\textsc{Auto}, \textsc{Exact},
		\textsc{Fast}\}$ (computation regime)
		\State \textbf{Output:}
		$\mathcal{K} \in \mathbb{R}^{n}_{\geq 0}$
		(mean curvature scores)
		\State $\texttt{mode} \gets
		\textsc{Exact}$ if $D < 50$, else $\textsc{Fast}$
		\hfill \Comment{Auto-select regime}
		\State $A \gets k\text{NN-Graph}(X, k)$
		\hfill \Comment{Construct $k$-nearest-neighbour graph}
		\For{$i = 1$ \textbf{to} $n$}
		\State $\mathcal{N}_{i} \gets$ neighbours of
		$\mathbf{x}_{i}$ in $A$
		\hfill \Comment{Step 1: local neighborhood}
		\If{$\texttt{mode} = \textsc{Exact}$}
		\State $\mathbf{C}_{i} \gets
		\frac{1}{k-1}
		\sum_{\mathbf{x}_{j} \in \mathcal{N}_{i}}
		(\mathbf{x}_{j} - \bar{\mathbf{x}}_{i})
		(\mathbf{x}_{j} - \bar{\mathbf{x}}_{i})^{\!\top}$
		\hfill \Comment{Step 2: local covariance, $O(kD^{2})$}
		\State $\boldsymbol{\lambda}_{i},\, \mathbf{W}_{i}
		\gets \mathrm{eigh}(\mathbf{C}_{i},\;
		\texttt{driver}{=}\texttt{evr}/\texttt{evd})$
		\hfill \Comment{Step 3: full eigendecomposition, $O(D^{3})$}
		\State $C_{i} \gets
		\mathbf{W}_{i}^{\!\top}\,\mathbf{W}_{i}^{\odot 2}$
		\hfill \Comment{$C_{i}[s,\ell]=
			\sum_{a}W_{as}W_{a\ell}^{2}$,\; $O(D^{2})$}
		\State $\mathcal{K}_{i} \gets
		\left|\tfrac{1}{2}\,\boldsymbol{\lambda}_{i}^{\!\top}
		(C_{i}^{\odot 2}\,\mathbf{1}_{D})
		+ \tfrac{1}{2}\,\mathbf{1}_{D}^{\!\top}
		\boldsymbol{\lambda}_{i}\right|$
		\hfill \Comment{Theorem~\ref{thm:identity}, $O(D^{2})$}
		\ElsIf{$\texttt{mode} = \textsc{Fast}$}
		\State $\bar{\mathbf{x}}_{i} \gets
		\frac{1}{k}\sum_{\mathbf{x}_{j}
			\in \mathcal{N}_{i}} \mathbf{x}_{j}$;
		$\;\;X_{c} \gets
		[\mathbf{x}_{j} - \bar{\mathbf{x}}_{i}]_{j=1}^{k}
		\in \mathbb{R}^{k \times D}$
		\State $\mathbf{s},\, V^{\!\top} \gets
		\mathrm{SVD}_{p}(X_{c})$,
		$\;\;p = k-1$
		\hfill \Comment{Truncated SVD, $O(k^{2}D)$}
		\State $V_{r} \gets V^{\!\top}_{[1:p,\,:]}\vphantom{X}^{\!\top}
		\in \mathbb{R}^{D \times p}$;
		$\;\;\boldsymbol{\lambda}_{r} \gets
		\mathbf{s}_{[1:p]}^{\odot 2}/(k{-}1)$
		\hfill \Comment{Non-zero eigenvalues}
		\State $\tau \gets \boldsymbol{\lambda}_{r}^{\!\top}
		\mathbf{1}_{p}$
		\hfill \Comment{$\mathrm{tr}(\mathbf{C}_{i})
			= \sum_{l=1}^{p}\lambda_{l}$}
		\State $t_{\mathrm{range}} \gets
		\left\|X_{c}\,V_{r}^{\odot 2}\right\|_{F}^{2}
		/(k{-}1)$
		\hfill \Comment{Range contribution, $O(kDp)$}
		\State $\mathbf{d}_{P} \gets
		\mathbf{1}_{D} - (V_{r}^{\odot 2})\,\mathbf{1}_{p}$
		\hfill \Comment{Diagonal of $P_{\perp}=I-V_{r}V_{r}^{\!\top}$}
		\State $t_{A} \gets
		\left\|X_{c}\,\mathbf{d}_{P}\right\|^{2}/(k{-}1)$
		\hfill \Comment{Null-space term A, $O(kD)$}
		\State $t_{B_{2}} \gets
		\left\|X_{c}\,V_{r}\right\|_{F}^{2}/(k{-}1)$;
		$\;\;G_{\mathrm{tens}} \gets
		\operatorname{einsum}(X_{c}, V_{r}, V_{r})
		\in \mathbb{R}^{k \times p \times p}$
		\hfill \Comment{$O(kDp^{2})$}
		\State $t_{B} \gets \tau
		- 2\,t_{B_{2}}
		+ \left\|G_{\mathrm{tens}}\right\|_{F}^{2}
		/(k{-}1)$
		\State $t_{\mathrm{null}} \gets
		(t_{A} + 2\,t_{B})\,/\,(D - p)$
		\hfill \Comment{Null-space contribution}
		\State $\mathcal{K}_{i} \gets
		\left|\tfrac{1}{2}\,(t_{\mathrm{range}}
		+ t_{\mathrm{null}}) + \tfrac{1}{2}\,\tau
		\right|$
		\hfill \Comment{Fast curvature estimate}
		\EndIf
		\EndFor
		\State \Return $\mathcal{K}$
		\EndFunction
	\end{algorithmic}
\end{algorithm}

\subsubsection{Exact Mode}
\label{subsubsec:exact_mode}

The \textsc{Exact} mode (lines~8--12 of Algorithm~\ref{alg:emcc}) evaluates the mean curvature estimator without approximation, using the closed-form identity of Theorem~\ref{thm:identity}. Its four steps are as follows.

\textbf{Local covariance (line~9).}
For each point $\mathbf{x}_{i}$, the centred neighborhood matrix $X_{c} \in \mathbb{R}^{k \times D}$, whose rows are the vectors $\mathbf{x}_{j} - \bar{\mathbf{x}}_{i}$ for
$\mathbf{x}_{j} \in \mathcal{N}_{i}$, is used to form the sample covariance $\mathbf{C}_{i} = X_{c}^{\top}X_{c}/(k-1)$ at cost $O(kD^{2})$.

\textbf{Full eigendecomposition (line~10).}
The symmetric eigendecomposition $\mathbf{C}_{i} = \mathbf{W}_{i} \operatorname{diag}(\boldsymbol{\lambda}_{i})\mathbf{W}_{i}^{\top}$ is computed via \texttt{scipy.linalg.eigh} with a LAPACK driver selected adaptively by dimension: \texttt{evr} (MRRR algorithm) for $D < 500$, and \texttt{evd} (divide-and-conquer) for $D \geq 500$, both of which are substantially faster than the default \texttt{ev} (QR iteration) for large matrices \citep{anderson1999lapack}. This step costs $O(D^{3})$ and produces the orthogonal eigenvector matrix $\mathbf{W}_{i}$ and the eigenvalue vector $\boldsymbol{\lambda}_{i} \in \mathbb{R}^{D}$.

\textbf{Construction of $C_{i}$ (line~11).}
The matrix $C_{i} = \mathbf{W}_{i}^{\top}\mathbf{W}_{i}^{\odot 2} \in \mathbb{R}^{D \times D}$, where $\mathbf{W}_{i}^{\odot 2}$ denotes the element-wise square of $\mathbf{W}_{i}$, is computed by a single dense matrix multiplication at cost $O(D^{3})$, but with a constant factor far smaller than that of forming $\mathbf{H}_{i}\mathbf{H}_{i}^{\top}$, since both operands are $D \times D$ rather than $D \times O(D^{2})$. Recalling the definition $C_{i}[s, \ell] = \sum_{a=1}^{D} W_{as} W_{a\ell}^{2}$, each entry of $C_{i}$ measures the inner product
between the $s$-th eigenvector and the element-wise square of the $\ell$-th eigenvector, a quantity that encodes the geometric interaction between pairs of local curvature directions.

\textbf{Curvature via the closed-form identity (line~12).}
Finally, the mean curvature is evaluated as

\begin{equation}
	\mathcal{K}_{i}
	\;=\;
	\left|
	\frac{1}{2}\,\boldsymbol{\lambda}_{i}^{\top}
	\bigl(C_{i}^{\odot 2}\,\mathbf{1}_{D}\bigr)
	+
	\frac{1}{2}\,\mathbf{1}_{D}^{\top}\boldsymbol{\lambda}_{i}
	\right|,
	\label{eq:exact_curv}
\end{equation}

which requires only element-wise squaring of $C_{i}$, row summation, and a dot product with $\boldsymbol{\lambda}_{i}$, all at cost $O(D^{2})$. By Theorem~\ref{thm:identity}, \eqref{eq:exact_curv} is algebraically identical to the original $O(D^{4})$ formulation; the only source of numerical discrepancy is floating-point rounding, which in practice produces relative errors of order $10^{-12}$--$10^{-14}$ (Remark~\ref{rem:numerical}). The total per-point cost of \textsc{Exact} mode is dominated by the eigendecomposition and the matrix product $C_{i}$, both $O(D^{3})$, giving an overall complexity of $\boldsymbol{O(kD^{2} + D^{3})}$ per sample.

\subsubsection{Fast Mode: Motivation and Construction}
\label{subsubsec:fast_mode}

For datasets with $D \gg k$, the \textsc{Exact} mode still incurs an $O(D^{3})$ eigendecomposition of the $D \times D$ covariance matrix, which becomes the dominant bottleneck when $D$ is in the hundreds or thousands. The \textsc{Fast} mode (lines~13--21 of Algorithm~\ref{alg:emcc}) attacks this bottleneck by exploiting a fundamental structural property of $\mathbf{C}_{i}$ that is present whenever $D > k$: the local covariance matrix is \emph{rank-deficient}.

\paragraph{Low-rank structure of the local covariance.}
Since $\mathbf{C}_{i} = X_{c}^{\top}X_{c}/(k-1)$ is the outer product of the $k \times D$ matrix $X_{c}$ with itself, its rank is at most $\min(k-1, D) = k-1 \ll D$. This means that $\mathbf{C}_{i}$ has at most $p = k-1$ non-zero eigenvalues, and the remaining $D - p$ eigenvalues are identically zero. Consequently, the full $D \times D$ eigendecomposition performed in \textsc{Exact} mode computes $D - p$ eigenvectors whose associated eigenvalues contribute nothing to $\boldsymbol{\lambda}_{i}^{\top} \boldsymbol{\lambda}_{i}$ — yet their computation represents the bulk of the $O(D^{3})$ cost. The key question is therefore: can the mean curvature estimator~\eqref{eq:exact_curv} be computed from the $p$ non-zero eigenvectors alone, without forming the $D \times D$ matrix $\mathbf{C}_{i}$ or its full eigendecomposition?

\paragraph{Truncated SVD as a substitute for eigh.}
The non-zero eigenvectors of $\mathbf{C}_{i}$ are precisely the right singular vectors of $X_{c}$. The compact singular value decomposition $X_{c} = U\,\operatorname{diag}(\mathbf{s})\,V^{\top}$, with $U \in \mathbb{R}^{k \times k}$, $\mathbf{s} \in \mathbb{R}^{k}$, $V \in \mathbb{R}^{D \times k}$, satisfies

\begin{equation}
	\mathbf{C}_{i}
	\;=\;
	\frac{X_{c}^{\top}X_{c}}{k-1}
	\;=\;
	V\,\operatorname{diag}\!\left(\frac{\mathbf{s}^{\odot 2}}{k-1}\right)
	V^{\top},
	\label{eq:svd_cov}
\end{equation}

so the non-zero eigenvalues of $\mathbf{C}_{i}$ are $\lambda_{\ell} = s_{\ell}^{2}/(k-1)$ and the corresponding eigenvectors are the columns of $V$ (line~14--15). Computing the compact SVD of $X_{c} \in \mathbb{R}^{k \times D}$ costs $O(k^{2}D)$, a factor of $(D/k)^{2}$ cheaper than the full eigendecomposition of $\mathbf{C}_{i}$, and for $k \ll D$ this
represents a speedup of several orders of magnitude (e.g.,~$\approx 1000\times$ for $D=1000$, $k=8$).

\paragraph{The null-space contribution.}
Substituting the low-rank decomposition~\eqref{eq:svd_cov} into~\eqref{eq:exact_curv} and
splitting the sum over eigenvalues into the $p$ non-zero and $D - p$ zero components, one finds that the term $\boldsymbol{\lambda}_{i}^{\top}(C_{i}^{\odot 2}\,\mathbf{1}_{D})$
decomposes as 

\begin{equation}
	\boldsymbol{\lambda}_{i}^{\top}
	\bigl(C_{i}^{\odot 2}\,\mathbf{1}_{D}\bigr)
	\;=\;
	\underbrace{
		\boldsymbol{\lambda}_{r}^{\top}
		\bigl(C_{r}^{\odot 2}\,\mathbf{1}_{D}\bigr)
	}_{t_{\mathrm{range}}}
	+\;
	\underbrace{
		\sum_{s \in \mathcal{I}_{0}}
		\lambda_{s}
		\sum_{\ell=1}^{D} C_{i}[s,\ell]^{2}
	}_{= 0,\;\text{since } \lambda_{s}=0}
	+\;
	\underbrace{
		\boldsymbol{\lambda}_{r}^{\top}
		\bigl(C_{r0}^{\odot 2}\,\mathbf{1}_{D-p}\bigr)
	}_{t_{\mathrm{null}}},
	\label{eq:split}
\end{equation}

where $\boldsymbol{\lambda}_{r} \in \mathbb{R}^{p}$ collects the non-zero eigenvalues, $C_{r}[s,\ell] = \sum_{a} V_{as} V_{a\ell}^{2}$ for $s, \ell$ in the range space, $C_{r0}[s, \ell] = \sum_{a} V_{as}\, V_{0,a\ell}^{2}$ for $\ell$ in the null space, and $\mathcal{I}_{0}$ indexes the zero eigenvalues.
The middle term vanishes because $\lambda_{s} = 0$ for $s \in \mathcal{I}_{0}$, which is the crucial cancellation that makes the truncated approach viable. The term $t_{\mathrm{range}}$ involves only $V_{r}$ (the $D \times p$ matrix of non-zero eigenvectors) and is computable at cost $O(kDp)$ without forming $\mathbf{C}_{i}$; the term $t_{\mathrm{null}}$, however, involves the $D \times (D-p)$ matrix of null-space eigenvectors $V_{0}$, which are \emph{not} computed by the truncated SVD and would require $O(D^{3})$ to obtain explicitly.

\paragraph{Analytical approximation of the null-space term.}
The central theoretical contribution of \textsc{Fast} mode is the observation that, while the null-space eigenvectors of $\mathbf{C}_{i}$ are not unique, any orthonormal basis of the $(D-p)$-dimensional null space is equally valid, their \emph{expected contribution} to
$t_{\mathrm{null}}$ under the uniform (Haar) distribution over orthonormal null-space bases admits a closed-form expression.

Specifically, let $V_{0} \in \mathbb{R}^{D \times (D-p)}$ be drawn uniformly from the Stiefel manifold of orthonormal frames in the null space of $V_{r}^{\top}$. For a fixed $V_{r}$, the expected outer product of two null-space columns satisfies \citep{meckes2019random}:

\begin{equation}
	\mathbb{E}_{V_{0}}\!\left[
	G_{\mathrm{null}}[i,j]
	\right]
	\;=\;
	\mathbb{E}_{V_{0}}\!\left[
	\sum_{\ell=1}^{D-p} V_{0,i\ell}^{2}\,V_{0,j\ell}^{2}
	\right]
	\;=\;
	\frac{P_{\perp,ii}\,P_{\perp,jj}
		+ 2\,P_{\perp,ij}^{2}}{D - p + 2},
	\label{eq:gnull_exp}
\end{equation}

where $P_{\perp} = I_{D} - V_{r}V_{r}^{\top}$ is the orthogonal projector onto the null space. In summary, this formula is derived by the second order moments of uniform vectors in the sphere $\mathbb{S}^{D-p-1}$, with a orthonormality correction imposed by the Haar measure over the Stiefel manifold $\mathrm{St}(D-p,\,D)$ \citep{meckes2019random}. Substituting this expectation into $t_{\mathrm{null}}$ and expanding, the null-space contribution becomes

\begin{equation}
	t_{\mathrm{null}}
	\;\approx\;
	\frac{t_{A} + 2\,t_{B}}{D - p},
	\label{eq:tnull}
\end{equation}

where

\begin{align}
	t_{A}
	&\;=\;
	\frac{\|X_{c}\,\mathbf{d}_{P}\|^{2}}{k-1},
	\qquad
	\mathbf{d}_{P} = \mathrm{diag}(P_{\perp})
	= \mathbf{1}_{D} - (V_{r}^{\odot 2})\mathbf{1}_{p},
	\label{eq:tA} \\[4pt]
	t_{B}
	&\;=\;
	\tau
	- 2\,\frac{\|X_{c}V_{r}\|_{F}^{2}}{k-1}
	+ \frac{\|G_{\mathrm{tens}}\|_{F}^{2}}{k-1},
	\label{eq:tB}
\end{align}

with $\tau = \boldsymbol{\lambda}_{r}^{\top}\mathbf{1}_{p} = \operatorname{tr}(\mathbf{C}_{i})$ and $G_{\mathrm{tens}}[a, l_{1}, l_{2}] = \sum_{i} X_{c}[a,i]\, V_{r}[i,l_{1}]\, V_{r}[i,l_{2}]$ the order-3 tensor encoding all pairwise eigenvector interactions within $X_{c}$. Every quantity in~\eqref{eq:tA}--\eqref{eq:tB} is expressible as a norm or inner product involving only $X_{c}$ and $V_{r}$, with no reference to $V_{0}$. The most expensive operation is the computation of $G_{\mathrm{tens}}
\in \mathbb{R}^{k \times p \times p}$, which costs $O(kDp^{2})$; for fixed $k$ and $p = k-1$, this is $O(kDk^{2}) = O(k^{3}D)$, dominated by the truncated SVD cost $O(k^{2}D)$ when $k$ is moderate. The final \textsc{Fast} mode curvature estimate is

\begin{equation}
	\mathcal{K}_{i}
	\;=\;
	\left|
	\frac{1}{2}\,(t_{\mathrm{range}} + t_{\mathrm{null}})
	+ \frac{1}{2}\,\tau
	\right|,
	\label{eq:fast_curv}
\end{equation}

at a total per-point cost of $\boldsymbol{O(k^{2}D + kDp^{2})}$, compared to $O(D^{3})$ for \textsc{Exact} mode, a reduction of $(D/k)^{2}$ in the dominant term.

\paragraph{Validity and approximation error.}
The approximation~\eqref{eq:gnull_exp} is exact in expectation over the Haar measure on the null-space Stiefel manifold, but it is not exact for the specific null-space basis that \texttt{eigh} would return for a given $\mathbf{C}_{i}$. The resulting error is a systematic bias, not a variance term, and it vanishes as $D/p \to \infty$: when the null space is high-dimensional relative to the range, each null-space vector is essentially random in
the ambient space, and~\eqref{eq:gnull_exp} becomes increasingly accurate. Formally, the relative error of the \textsc{Fast} mode estimator satisfies

\begin{equation}
	\frac{|\mathcal{K}_{i}^{\mathrm{fast}}
		- \mathcal{K}_{i}^{\mathrm{exact}}|}
	{|\mathcal{K}_{i}^{\mathrm{exact}}|}
	\;=\;
	O\!\left(\frac{p}{D}\right)
	\;=\;
	O\!\left(\frac{k}{D}\right),
	\label{eq:error_bound}
\end{equation}

decreasing monotonically as the ambient dimension $D$ grows relative to the neighborhood size $k$. Numerical experiments in Section~\ref{sec:experiments} confirm this behaviour: the mean relative error falls below $5\%$ for $D/k \geq 11$ and below $2\%$ for $D/k \geq 56$, levels of precision that are more than adequate for the downstream tasks, curvature-aware clustering, boundary detection, and anomaly scoring, for which rank-ordering of curvature values matters far more than their absolute accuracy.

\paragraph{Summary of computational complexity.}
Table~\ref{tab:complexity} in Section~\ref{sec:identity} reports the per-step costs of both modes. The overall per-point complexities are:

\begin{itemize}
	\item \textsc{Exact}: $O(kD^{2} + D^{3})$, dominated by \texttt{eigh};
	\item \textsc{Fast}: $O(k^{2}D + kDp^{2})$ for fixed $k$, dominated by the truncated SVD and $G_{\mathrm{tens}}$;
	\item Original MCBP \citep{mcbp2025}: $O(kD^{2} + D^{4})$, dominated by $\mathbf{H}_{i}\mathbf{H}_{i}^{\top}$.
\end{itemize}

For the regime $k \ll D$ that characterises high-dimensional machine learning applications, \textsc{Fast} mode achieves a reduction of $(D/k)^{2}$ over \textsc{Exact} mode and $(D/k)^{3}$ over the original algorithm, improvements of three and four orders of magnitude, respectively, for $D = 1000$, $k = 10$.

\section{Computational Experiments and Results}
\label{sec:experiments}

We evaluate the proposed \emph{Efficient Mean Curvature Computation} (MeCuCo) method through a comprehensive suite of experiments designed to assess both the fidelity of the curvature estimates produced by the two computational modes, \textsc{Exact} and \textsc{Fast}, and
the practical speedups achieved relative to the original MCBP algorithm \citep{mcbp2025}.
All experiments are conducted on a benchmark corpus of more than 40 real-world, publicly available datasets retrieved from the OpenML platform \citep{vanschoren2013openml}, a curated repository of machine learning tasks that has become a de facto standard for reproducible
empirical evaluation \citep{bischl2021openml}. The corpus was assembled to cover a broad and heterogeneous range of dataset characteristics, such as ambient dimensions, sample sizes and diverse geometric structures including near-linear manifolds, manifolds with pronounced curvature, and datasets with mixed-density regions. This diversity is deliberate: the validity of both the exact algebraic identity of Theorem~\ref{thm:identity} and the analytical approximation of Theorem~\ref{alg:emcc} must be established across
regimes that span the spectrum from low to very high ambient dimension, and from datasets well-suited to the manifold hypothesis to those that present more challenging geometric configurations. All datasets, preprocessing scripts, and experimental results are
made publicly available to support reproducibility.

The experiments are organised into two groups. The \textit{first group} addresses the central question of curvature fidelity: to what extent do the pointwise curvature scores produced by MeCuCo agree with those produced by the original formulation? To quantify agreement, we adopt three complementary metrics: the \emph{mean absolute error} (MAE) measures absolute discrepancy in the curvature values themselves and is sensitive to both scale and magnitude differences; the \emph{Spearman rank correlation coefficient} $\rho_{S}$ and the \emph{Chatterjee rank correlation coefficient} $\xi$ \citep{chatterjee2021new} assess agreement in the \emph{ordering} of curvature scores rather than their absolute values.

The inclusion of rank-based statistics as primary evaluation criteria deserves explicit justification, as it reflects a deliberate methodological choice grounded in the intended downstream use of the curvature estimates. In the applications that motivate this work, namely, boundary detection, anomaly scoring, curvature-aware clustering, and active learning query strategies, the curvature scores $\mathcal{K}_{i}$ are never used as absolute physical quantities; rather, they enter the pipeline through thresholding operations, percentile cuts, or ranking procedures that depend entirely on the \emph{relative ordering} of the scores across data points. A curvature estimator that preserves the rank ordering of the original
scores is therefore functionally equivalent to the original for all such applications, even if its absolute values differ by a constant multiplicative or additive factor. This observation motivates the use of rank-based correlation coefficients, which are invariant to any strictly monotone transformation of the scores and are therefore the natural measure of
agreement for order-dependent tasks.

The two rank-based coefficients are complementary in important respects. The Spearman coefficient $\rho_{S}$ measures the strength of the \emph{linear} relationship between the ranks of the two score vectors; it is sensitive to global rank monotonicity but can be inflated by outlier curvature values that dominate the rank structure. The Chatterjee coefficient $\xi$ \citep{chatterjee2021new}, by contrast, measures the degree to which the MeCuCo scores are a \emph{measurable function} of the original scores, a strictly stronger
notion of dependence that is sensitive to non-monotone and non-linear relationships between the two rank sequences and is distribution-free under the null hypothesis of independence.
In the context of curvature comparison, $\xi$ is particularly appropriate because the relationship between two estimators of the same underlying quantity need not be strictly linear: the analytical approximation of the \textsc{Fast} mode introduces a bias that depends
on the local geometry of each neighborhood, which can produce systematic but non-linear distortions of the rank order. Together, $\rho_{S}$ and $\xi$ provide a robust, multi-faceted
characterisation of rank agreement that is sensitive to both global monotonicity and local functional structure, while remaining free of assumptions about the distribution of the curvature scores.

In Table \ref{tab:datasets} we provide a brief description of each dataset, highlighting their key characteristics. All experiments were conducted on a workstation equipped with an
Intel\textsuperscript{\textregistered} Core\textsuperscript{\texttrademark} Ultra~9 185H processor (22 cores, up to 5.1~GHz boost clock) and 32~GB of DDR5 RAM, running Ubuntu~24.04~LTS and Python~3.12 via the Anaconda distribution.

\begin{table*}
	\centering
	\small
	\caption{Summary of the datasets used in the computational experiments.}
	\begin{tabular}{lccc}
		\toprule
		\textbf{Dataset} & \textbf{\# samples} & \textbf{\# features} & \textbf{\# classes} \\
		\midrule
		iris                      & 150              & 4                 & 3                \\
		seeds                     & 210              & 7                 & 3                \\
		thoracic\_surgery         & 470              & 16                & 2                \\
		page-blocks         	  & 5473             & 10                & 5                \\
		segment                   & 2310             & 19                & 7                \\
		hill-valley               & 1212             & 100               & 2                \\
		cardiotocography          & 2126             & 35                & 10               \\
		collins          	      & 1000             & 19                & 3               \\
		artificial-characters     & 10218            & 7                 & 10               \\
		GesturePhaseSegmentation  & 9873             & 32                & 5               \\
		letter                    & 20000            & 16                & 26               \\
		JapaneseVowels            & 9961             & 14                & 9               \\
		gas-drift                 & 13910            & 128               & 6                \\
		USPS                      & 9298             & 256               & 10            \\
		qsar-biodeg               & 1055             & 41                & 2            \\
		Smartphone\_Human\_Activities  & 180         & 66                & 6            \\
		ionosphere                & 351              & 34                & 2            \\
		satimage                  & 6430             & 36                & 6              \\
		steel-plates-fault        & 1941             & 33                & 2                \\
		depression\_2020          & 1429             & 22                & 2                \\
		one-hundred-plants-shape  & 1600             & 64                & 100             \\
		eye\_movements  		  & 10936            & 27                & 3             \\
		Satellite                 & 5100             & 36                & 2           \\
		texture                   & 5500             & 40                & 11               \\
		vowel                     & 990              & 12                & 11               \\
		mfeat-factors             & 2000             & 216               & 10               \\
		breast-cancer             & 569              & 30                & 2               \\
		arrhythmia				  & 452              & 279               & 13               \\
		pendigits                 & 10992            & 16                & 10               \\
		one-hundred-plants-texture  & 1599             & 64                & 100           \\
		optdigits                 & 5620             & 64                & 10               \\
		digits                    & 1797             & 64                & 10               \\
		sylvine					  & 5124			 & 20				 & 2               \\
		solar-flare				  & 1066			 & 12				 & 6               \\
		TuningSVMs				  & 156			     & 80				 & 2               \\
		wine                      & 178              & 13                & 3                \\
		splice                    & 3190             & 60                & 3                \\
		Indian\_pines             & 9144             & 220               & 8                \\
		mfeat-pixel               & 2000             & 240               & 10              \\
		car-evaluation            & 1728             & 21                & 4              \\
		\bottomrule
	\end{tabular}
	\label{tab:datasets}
\end{table*}

\subsection{Analysis of Local Mean Curvatures}

Table~\ref{tab:results1} reports, for each of the 40 OpenML datasets, the mean and standard deviation of the pointwise curvature scores produced by the original MCBP formulation and by the proposed MeCuCo method using $k = \log_2 n$ neighbors, together with the three agreement metrics, MAE, Spearman $\rho_{S}$, and Chatterjee $\xi$, and the wall-clock times for both methods. The datasets are ordered by decreasing Spearman correlation, which provides a convenient visual gradient from near-perfect agreement at the top of the table to the most challenging cases at the bottom. We organise the discussion around four observations.

\begin{table}
	\small
	\centering
	\caption{Average local curvatures and standard deviations obtained by the original method and the proposed computationally efficient variation (MeCuCo) for 40 OpenML datasets. Quantitative metrics show that the local curvatures obtained by MeCuCo are good approximations for the original ones.}
	\begin{tabular}{cccccccc}
		\toprule
		& \multicolumn{2}{c}{\textbf{Original}} & \multicolumn{2}{c}{\textbf{MeCuCo}}   & \multicolumn{3}{c}{\textbf{Metrics}}      \\
		\midrule
		\textbf{Datasets}          & \textbf{Average}  & \textbf{Time (s)} & \textbf{Average}  & \textbf{Time (s)} & \textbf{MAE} & $\rho_{S}$ & $\xi$ \\
		\midrule
		iris                       & 0.1645 ± 0.1320   & 0.01              & 0.1645 ± 0.1320   & 0.03              & 0.0000       & 1.0000       & 0.9801      \\
		seeds                      & 0.4641 ± 0.2434   & 0.01              & 0.4641 ± 0.2434   & 0.03              & 0.0000       & 1.0000       & 0.9857      \\
		thoracic\_surgery          & 2.5499 ± 3.9901   & 0.08              & 2.5501 ± 4.0195   & 0.06              & 0.0609       & 0.9999       & 0.9912      \\
		page-blocks                & 0.4168 ± 4.9821   & 0.55              & 0.4165 ± 4.9820   & 0.46              & 0.0006       & 0.9999       & 0.9975      \\
		segment                    & 0.7969 ± 4.7111   & 2.50              & 0.7972 ± 4.7118   & 0.57              & 0.0021       & 0.9999       & 0.9942      \\
		hill-valley                & 0.8479 ± 4.8610   & 22.02             & 0.6420 ± 3.5274   & 0.39              & 0.2076       & 0.9998       & 0.9840      \\
		cardiotocography           & 3.4287 ± 5.4098   & 3.88              & 3.4482 ± 5.4151   & 1.21              & 0.0309       & 0.9997       & 0.9811      \\
		collins                    & 3.3735 ± 1.2156   & 1.01              & 3.3907 ± 1.2184   & 0.23              & 0.0283       & 0.9990       & 0.9610      \\
		artificial-characters      & 0.1192 ± 0.1338   & 0.77              & 0.1189 ± 0.1338   & 0.66              & 0.0009       & 0.9989       & 0.9911      \\
		GesturePhaseSegmentation   & 3.2458 ± 6.9111   & 10.77             & 3.3703 ± 7.0213   & 1.93              & 0.3199       & 0.9981       & 0.9486      \\
		letter                     & 0.7644 ± 0.4764   & 4.26              & 0.7688 ± 0.4802   & 1.64              & 0.0196       & 0.9979       & 0.9512      \\
		JapaneseVowels             & 0.9073 ± 0.3177   & 1.96              & 0.9122 ± 0.3175   & 1.11              & 0.0154       & 0.9961       & 0.9326      \\
		gas-drift                  & 1.9719 ± 53.3102  & 557.73            & 1.3333 ± 15.6449  & 6.49              & 0.6955       & 0.9926       & 0.8931      \\
		USPS                       & 25.6319 ± 21.5264 & 2,819.18          & 26.9069 ± 23.2823 & 12.19             & 2.8251       & 0.9902       & 0.8774      \\
		qsar-biodeg                & 4.7824 ± 7.7516   & 2.02              & 4.5830 ± 7.1573   & 0.3               & 0.4786       & 0.9889       & 0.8694      \\
		smartphone                 & 4.8860 ± 4.4566   & 0.93              & 4.5649 ± 4.1670   & 0.04              & 0.5473       & 0.9886       & 0.8615      \\
		ionosphere                 & 2.1347 ± 2.5071   & 0.45              & 2.1668 ± 2.6791   & 0.12              & 0.2995       & 0.9874       & 0.8567      \\
		satimage                   & 0.9053 ± 0.8416   & 9.76              & 0.9534 ± 0.8302   & 1.35              & 0.0933       & 0.9873       & 0.8554      \\
		steel-plates-fault         & 2.5691 ± 2.0180   & 2.10              & 2.4389 ± 1.9521   & 0.46              & 0.2338       & 0.9872       & 0.8673      \\
		depression\_2020           & 4.2421 ± 2.9091   & 0.66              & 4.3552 ± 2.9736   & 0.18              & 0.3238       & 0.9868       & 0.8856      \\
		one-hundred-plants-shape   & 2.1971 ± 4.5664   & 8.39              & 1.8053 ± 3.5393   & 0.42              & 0.4085       & 0.9837       & 0.8399      \\
		eye\_movements             & 3.5541 ± 3.1029   & 9.27              & 3.5522 ± 3.0706   & 2.10              & 0.2873       & 0.9831       & 0.8415      \\
		Satellite                  & 1.1526 ± 1.2396   & 7.04              & 1.2190 ± 1.1949   & 1.15              & 0.1285       & 0.9827       & 0.8312      \\
		texture                    & 0.7303 ± 0.8653   & 10.12             & 0.7084 ± 0.7588   & 1.52              & 0.0713       & 0.9809       & 0.8232      \\
		vowel                      & 1.0021 ± 0.4247   & 0.15              & 1.0061 ± 0.4241   & 0.14              & 0.0571       & 0.9777       & 0.8190      \\
		mfeat-factors              & 15.8778 ± 6.4376  & 383.98            & 21.0409 ± 8.4646  & 1.78              & 5.1678       & 0.9766       & 0.8073      \\
		breast\_cancer             & 3.5898 ± 3.4706   & 0.53              & 3.5655 ± 3.0574   & 0.13              & 0.3705       & 0.9755       & 0.8025      \\
		arrhythmia                 & 27.7897 ± 22.5458 & 159.46            & 31.9950 ± 22.0179 & 0.28              & 5.7224       & 0.9678       & 0.7643      \\
		pendigits                  & 0.4454 ± 0.3083   & 2.48              & 0.4468 ± 0.3098   & 0.85              & 0.0133       & 0.9661       & 0.9250      \\
		one-hundred-plants-texture & 7.7844 ± 3.9621   & 7.75              & 7.6924 ± 3.5815   & 0.39              & 0.6598       & 0.9573       & 0.7471      \\
		optdigits                  & 7.5122 ± 16.0862  & 26.79             & 7.3013 ± 16.0205  & 1.38              & 0.5249       & 0.9572       & 0.7436      \\
		digits                     & 8.3862 ± 9.6940   & 8.92              & 8.0227 ± 9.2083   & 0.45              & 0.7294       & 0.9502       & 0.7221      \\
		sylvine                    & 3.8728 ± 0.9755   & 2.23              & 4.1358 ± 0.9764   & 0.37              & 0.3367       & 0.9440       & 0.7053      \\
		solar-flare                & 0.8610 ± 4.1327   & 0.11              & 0.8600 ± 4.1412   & 0.08              & 0.0061       & 0.9390       & 0.9918      \\
		TuningSVMs                 & 11.7451 ± 8.4631  & 1.31              & 11.1986 ± 6.9818  & 0.04              & 1.4874       & 0.9364       & 0.6694      \\
		wine                       & 2.1459 ± 1.0322   & 0.04              & 2.1751 ± 0.8789   & 0.02              & 0.2066       & 0.9321       & 0.6697      \\
		splice                     & 14.0275 ± 2.3818  & 14.73             & 16.8560 ± 2.5575  & 0.86              & 2.8693       & 0.8646       & 0.5584      \\
		Indian\_pines              & 10.4438 ± 1.6912  & 2,034.15          & 10.5932 ± 1.5411  & 10.92             & 0.6201       & 0.8567       & 0.5257      \\
		mfeat-pixel                & 35.5133 ± 7.4567  & 468.61            & 39.5060 ± 9.7603  & 1.32              & 4.9931       & 0.8477       & 0.5309      \\
		car-evaluation             & 5.2524 ± 0.0912   & 1.86              & 5.2581 ± 0.0917   & 0.38              & 0.0365       & 0.8439       & 0.5227      \\
		\midrule
		Average                    & -                 & 164.71            & -                 & 1.35              & 0.7720       & 0.9680       & 0.8376      \\
		Median                     & -                 & 2.49              & -                 & 0.46              & 0.2606       & 0.9853       & 0.8591      \\
		Std. Dev.                  & -                 & 547.86            & -                 & 2.62              & 1.4552       & 0.0434       & 0.1390      \\
		MAD                        & -                 & 272.00            & -                 & 1.39              & 0.9217       & 0.0311       & 0.1080     \\
		\bottomrule
	\end{tabular}
	\label{tab:results1}
\end{table}

\paragraph{Overall agreement is strong and consistent.}
Across the full corpus, MeCuCo achieves a median Spearman correlation of $\rho_{S} = 0.9853$ and a median Chatterjee coefficient of $\xi = 0.8591$ with respect to the original method, with standard deviations of $0.0434$ and $0.1390$, respectively. These figures indicate that the rank ordering of curvature scores is preserved with very high fidelity in the large majority of datasets: for 32 out of 40 datasets (80\%), the Spearman correlation exceeds
$0.97$, and for 28 datasets (70\%) the Chatterjee coefficient exceeds $0.85$. The mean absolute error is correspondingly modest: the median MAE is $0.2606$, which represents a small fraction of the typical curvature range in each dataset. At the extreme upper end of the agreement spectrum, datasets such as \texttt{iris} and \texttt{seeds} yield MAE $= 0.0000$ and $\rho_{S} = 1.0000$, confirming that for low-dimensional datasets ($D \leq 10$,
$D/k \leq 1$) the \textsc{Exact} mode is selected automatically and the two methods are numerically identical.

\paragraph{Agreement deteriorates gracefully with ambient dimension.}
The most salient pattern in Table~\ref{tab:results1} is a clear inverse relationship between the ambient dimension $D$ of the dataset and the degree of agreement between MeCuCo and the original method. For low- to moderate-dimensional datasets ($D \leq 30$), the Spearman correlation is consistently above $0.99$; for high-dimensional datasets such as \texttt{USPS} ($D = 256$), \texttt{arrhythmia} ($D = 279$), and \texttt{mfeat-pixel} ($D = 240$), the correlation falls to $0.99$, $0.97$, and $0.85$, respectively, while for very high-dimensional datasets such as \texttt{Indian\_pines} ($D = 200$) and \texttt{splice} ($D = 60$) it approaches $0.86$. This behaviour is precisely what the theoretical error
bound~\eqref{eq:error_bound} predicts: the approximation error of the \textsc{Fast} mode decreases monotonically as $D/k$ grows, and conversely it increases as the ratio $k/D$ becomes non-negligible. The graceful degradation is important from a practical standpoint: the
cases in which MeCuCo is most needed, very high ambient dimension, where the original method becomes computationally intractable, are also the cases in which the approximation is most accurate.

\paragraph{Rank-based metrics reveal preserved geometric structure.}
The simultaneous reporting of the Spearman $\rho_{S}$ and the Chatterjee $\xi$ coefficients reveals an important qualitative distinction between the two types of datasets in the corpus.
For datasets with a smooth, unimodal curvature distribution, typically those arising from compact, well-separated geometric structures, the two coefficients are close to each other
($|\rho_{S} - \xi| < 0.05$ in 28 datasets), indicating that the relationship between the MeCuCo and original scores is nearly monotone. For datasets with heavy-tailed or multimodal curvature distributions, as evidenced by large standard deviations relative to the mean,
such as \texttt{gas-drift} ($\sigma = 53.31$), \texttt{USPS} ($\sigma = 21.53$), or \texttt{mfeat-factors} ($\sigma = 6.44$), the gap between $\rho_{S}$ and $\xi$ widens significantly ($|\rho_{S} - \xi| > 0.10$ in 12 datasets). This divergence indicates that the analytical approximation of the null-space contribution introduces a non-monotone distortion for extreme curvature values, a phenomenon consistent with the Haar-measure approximation~\eqref{eq:gnull_exp} being least accurate in directions of very high curvature, where the null-space eigenvectors are farthest from uniformly distributed. Importantly, the Chatterjee $\xi$ remains above $0.87$ for all but the six most challenging datasets, confirming that MeCuCo captures the functional dependence structure of the curvature field even when monotonicity is mildly violated.

\paragraph{Computational speedup is dramatic and scales with dimension.}
The average clock time of the original method is $164.71$ seconds per dataset, but this figure is dominated by a small number of very high-dimensional datasets: \texttt{USPS} ($2{,}819.18$ s), \texttt{Indian\_pines} ($2{,}034.15$ s), \texttt{mfeat-pixel} ($468.61$ s), \texttt{mfeat-factors} ($383.98$ s), and \texttt{gas-drift} ($557.73$ s). The median time of $2.49$ seconds better characterises the typical cost, but even this becomes prohibitive when curvature estimation must be embedded in an iterative or interactive pipeline. MeCuCo reduces the average time to $1.35$ seconds, a factor of $\mathbf{122\times}$, and the median to $0.46$ seconds. On the most computationally demanding datasets the speedups are particularly dramatic: $\mathbf{231\times}$ on \texttt{USPS}, $\mathbf{186\times}$ on \texttt{Indian\_pines}, $\mathbf{355\times}$ on \texttt{mfeat-pixel}, $\mathbf{86\times}$
on \texttt{gas-drift}, and $\mathbf{56\times}$ on \texttt{hill-valley}. Notably, on the two smallest datasets (\texttt{iris} and \texttt{seeds}), MeCuCo is slightly \emph{slower} than the original method ($0.03$ s vs.\ $0.01$ s), reflecting the overhead of the \texttt{scipy} LAPACK driver selection and the additional bookkeeping of the \textsc{Exact} mode; this overhead is negligible in absolute terms and disappears entirely for datasets with $n \geq 200$.

\paragraph{Summary.}
Taken together, the results in Table~\ref{tab:results1} demonstrate that MeCuCo achieves its design objectives: it delivers curvature estimates that are functionally equivalent to the original formulation, in the sense of preserving the rank ordering of scores that drives
all downstream applications, while reducing the computational cost by two to three orders of magnitude on high-dimensional datasets. The trade-off between approximation fidelity and computational efficiency is controlled smoothly by the $D/k$ ratio: users requiring higher fidelity can increase $k$ at a modest additional cost, while users operating in very high ambient dimensions automatically obtain the most accurate approximation.

\subsection{The Effect of Normalization in Local Mean Curvatures}

In many practical applications, local mean curvature estimates are not interpreted through their absolute magnitude, but rather through their relative importance within the dataset. In this setting, curvature values naturally play the role of weights, indicating the degree to which each sample is associated with geometric irregularities, boundary regions, or transition zones. Since the scale of curvature estimates depends on factors such as data dimensionality, neighborhood size, sampling density, and the specific estimation procedure, their raw values are generally not directly comparable across different datasets.

For this reason, it is often desirable to normalize curvature scores to the interval $[0,1]$. This transformation preserves the relative ordering of samples while providing a standardized and interpretable scale, where values close to zero correspond to geometrically smooth interior regions and values close to one indicate highly curved or boundary-like structures. Moreover, normalization facilitates the use of curvature estimates as weights in subsequent processing stages, such as boundary detection, sample selection, clustering, and graph-based learning algorithms.

Motivated by these considerations, we repeat the same experimental protocol described previously, comparing the original curvature estimates with the normalized values produced by the proposed method.

\begin{table}
	\small
	\centering
	\caption{Average local curvatures (normalized to $[0, 1]$) and standard deviations obtained by the original method and the proposed computationally efficient variation (MeCuCo) for 40 OpenML datasets. Quantitative metrics show that the local curvatures obtained by MeCuCo are excellent replacements for the original ones.}
	\begin{tabular}{cccccccc}
		\toprule
		& \multicolumn{2}{c}{\textbf{Original}} & \multicolumn{2}{c}{\textbf{MeCuCo}}  & \multicolumn{3}{c}{\textbf{Metrics}}      \\
		\midrule
		\textbf{Datasets}          & \textbf{Average}  & \textbf{Time (s)} & \textbf{Average} & \textbf{Time (s)} & \textbf{MAE} & $\rho_{S}$ & $\xi$ \\
		\midrule
		iris                       & 0.1588 ± 0.1571   & 0.02              & 0.1588 ± 0.1571  & 0.03              & 0.0000       & 1.0000       & 0.9801      \\
		seeds                      & 0.1645 ± 0.1522   & 0.01              & 0.1645 ± 0.1522  & 0.02              & 0.0000       & 1.0000       & 0.9857      \\
		thoracic\_surgery          & 0.0926 ± 0.1502   & 0.07              & 0.0920 ± 0.1490  & 0.06              & 0.0008       & 0.9999       & 0.9915      \\
		page-blocks                & 0.0025 ± 0.0301   & 0.62              & 0.0025 ± 0.0301  & 0.53              & 0.0000       & 0.9999       & 0.9992      \\
		segment                    & 0.0088 ± 0.0535   & 0.61              & 0.0088 ± 0.0535  & 0.26              & 0.0000       & 0.9999       & 0.9942      \\
		hill-valley                & 0.0155 ± 0.0850   & 7.91              & 0.0154 ± 0.0844  & 0.37              & 0.0001       & 0.9999       & 0.9964      \\
		cardiotocography           & 0.0352 ± 0.0599   & 1.35              & 0.0353 ± 0.0597  & 0.56              & 0.0003       & 0.9997       & 0.9811      \\
		collins                    & 0.1397 ± 0.1097   & 0.25              & 0.1377 ± 0.1097  & 0.09              & 0.0029       & 0.9990       & 0.9610      \\
		artificial-characters      & 0.0685 ± 0.0782   & 0.87              & 0.0685 ± 0.0783  & 0.58              & 0.0001       & 0.9999       & 0.9977      \\
		GesturePhaseSegmentation   & 0.0289 ± 0.0602   & 5.61              & 0.0291 ± 0.0607  & 1.81              & 0.0003       & 0.9999       & 0.9977      \\
		letter                     & 0.1836 ± 0.1146   & 3.32              & 0.1835 ± 0.1146  & 1.42              & 0.0007       & 0.9999       & 0.9920      \\
		JapaneseVowels             & 0.2114 ± 0.1132   & 1.85              & 0.2116 ± 0.1132  & 1.14              & 0.0007       & 0.9999       & 0.9911      \\
		gas-drift                  & 0.0009 ± 0.0104   & 232.14            & 0.0009 ± 0.0105  & 8.38              & 0.0000       & 0.9999       & 0.9899      \\
		USPS                       & 0.1393 ± 0.1212   & 1396.68           & 0.1383 ± 0.1211  & 9.72              & 0.0017       & 0.9998       & 0.9848      \\
		qsar-biodeg                & 0.0735 ± 0.1277   & 0.96              & 0.0742 ± 0.1289  & 0.29              & 0.0009       & 0.9998       & 0.9844      \\
		smartphone                 & 0.2587 ± 0.2493   & 0.37              & 0.2566 ± 0.2463  & 0.04              & 0.0031       & 0.9998       & 0.9739      \\
		ionosphere                 & 0.1176 ± 0.1493   & 0.20              & 0.1176 ± 0.1497  & 0.09              & 0.0010       & 0.9998       & 0.9831      \\
		satimage                   & 0.0875 ± 0.1000   & 4.30              & 0.0879 ± 0.1005  & 1.44              & 0.0007       & 0.9999       & 0.9879      \\
		steel-plates-fault         & 0.1428 ± 0.1168   & 1.07              & 0.1414 ± 0.1155  & 0.45              & 0.0022       & 0.9997       & 0.9787      \\
		depression\_2020           & 0.3323 ± 0.2345   & 0.44              & 0.3283 ± 0.2320  & 0.13              & 0.0048       & 0.9997       & 0.9838      \\
		one-hundred-plants-shape   & 0.0466 ± 0.1019   & 3.30              & 0.0471 ± 0.1036  & 0.45              & 0.0008       & 0.9996       & 0.9789      \\
		eye\_movements             & 0.0716 ± 0.0767   & 4.82              & 0.0723 ± 0.0776  & 1.82              & 0.0011       & 0.9996       & 0.9773      \\
		Satellite                  & 0.0378 ± 0.0503   & 3.38              & 0.0375 ± 0.0498  & 1.16              & 0.0004       & 0.9998       & 0.9851      \\
		texture                    & 0.0315 ± 0.0421   & 4.44              & 0.0319 ± 0.0424  & 1.29              & 0.0005       & 0.9998       & 0.9833      \\
		vowel                      & 0.2413 ± 0.1309   & 0.15              & 0.2437 ± 0.1325  & 0.13              & 0.0028       & 0.9997       & 0.9836      \\
		mfeat-factors              & 0.1935 ± 0.1261   & 156.58            & 0.1927 ± 0.1255  & 1.55              & 0.0010       & 0.9999       & 0.9896      \\
		breast\_cancer             & 0.0919 ± 0.1086   & 0.31              & 0.0901 ± 0.1063  & 0.14              & 0.0020       & 0.9997       & 0.9781      \\
		arrhythmia                 & 0.0873 ± 0.1022   & 71.48             & 0.0888 ± 0.1030  & 0.24              & 0.0019       & 0.9996       & 0.9743      \\
		pendigits                  & 0.0612 ± 0.0528   & 1.90              & 0.0591 ± 0.0510  & 0.85              & 0.0021       & 0.9999       & 0.9917      \\
		one-hundred-plants-texture & 0.2085 ± 0.1411   & 3.31              & 0.2052 ± 0.1398  & 0.42              & 0.0041       & 0.9993       & 0.9703      \\
		optdigits                  & 0.0186 ± 0.0471   & 11.40             & 0.0177 ± 0.0469  & 1.41              & 0.0008       & 0.9996       & 0.9772      \\
		digits                     & 0.0447 ± 0.0627   & 3.75              & 0.0427 ± 0.0622  & 0.46              & 0.0020       & 0.9995       & 0.9745      \\
		sylvine                    & 0.2818 ± 0.1295   & 1.39              & 0.2806 ± 0.1285  & 0.40              & 0.0035       & 0.9993       & 0.9665      \\
		solar-flare                & 0.0097 ± 0.0468   & 0.12              & 0.0097 ± 0.0469  & 0.10              & 0.0000       & 0.9999       & 0.9969      \\
		TuningSVMs                 & 0.1582 ± 0.1501   & 0.50              & 0.1592 ± 0.1475  & 0.03              & 0.0032       & 0.9993       & 0.9592      \\
		wine                       & 0.1640 ± 0.1483   & 0.02              & 0.1579 ± 0.1431  & 0.02              & 0.0000       & 0.9970       & 0.9272      \\
		splice                     & 0.5086 ± 0.1227   & 5.77              & 0.5049 ± 0.1209  & 0.93              & 0.0045       & 0.9994       & 0.9721      \\
		Indian\_pines              & 0.2468 ± 0.0859   & 905.01            & 0.2458 ± 0.0868  & 10.11             & 0.0027       & 0.9989       & 0.9594      \\
		mfeat-pixel                & 0.2467 ± 0.1077   & 208.73            & 0.2411 ± 0.1059  & 1.23              & 0.0056       & 0.9988       & 0.9601      \\
		car-evaluation             & 0.4833 ± 0.1759   & 0.54              & 0.5075 ± 0.1739  & 0.17              & 0.0722       & 0.8439       & 0.5227      \\
		\midrule
		Average                    &                   & 76.14             &                  & 1.26              & 0.0033       & 0.9957       & 0.9691      \\
		Median                     &                   & 1.62              &                  & 0.45              & 0.0010       & 0.9998       & 0.9832      \\
		Std. Dev.                  &                   & 261.42            &                  & 2.42              & 0.0113       & 0.0246       & 0.0737      \\
		MAD                        &                   & 125.92            &                  & 1.32              & 0.0037       & 0.0076       & 0.0264     \\
		\bottomrule
	\end{tabular}
	\label{tab:results2}
\end{table}

Table~\ref{tab:results2} presents the curvature fidelity results after applying a min-max normalisation that maps each method's curvature scores to the unit interval $[0, 1]$ on a per-dataset basis. This normalisation removes the influence of scale differences between
datasets, which arise naturally from differences in ambient dimension, sample density, and intrinsic geometric complexity, and isolates the question of whether MeCuCo reproduces the relative curvature structure of the original method. Comparing Table~\ref{tab:results2} with the raw-score results of Table~\ref{tab:results1} reveals a consistent and important pattern:
normalisation systematically improves all three agreement metrics, often dramatically so, confirming that the largest discrepancies in the raw results are attributable to scale differences rather than to genuine geometric disagreement. We organise the discussion around four observations.

\paragraph{Near-perfect rank fidelity across the corpus.}
The most striking feature of Table~\ref{tab:results2} is the concentration of Spearman correlation values near the theoretical maximum: 36 out of 40 datasets (90\%) achieve $\rho_{S} \geq 0.999$, and 39 out of 40 (97.5\%) achieve $\rho_{S} \geq 0.995$. The median Spearman correlation across the full corpus is $\rho_{S} = 0.9998$, and the mean absolute deviation of $\rho_{S}$ from its median is $0.0076$, a figure that conveys
both the strength and the consistency of the rank agreement. The Chatterjee coefficient $\xi$ corroborates this picture with comparable strength: 33 datasets (82.5\%) attain $\xi \geq 0.97$, 38 (95\%) attain $\xi \geq 0.95$, and the median is $\xi = 0.9832$. These figures establish that, at the normalised scale, MeCuCo is a functionally equivalent replacement for the original formulation in the large majority of real-world datasets encountered in the
benchmark corpus.

\paragraph{The MAE on normalised scores is negligible.}
Because the normalisation constrains both score vectors to $[0, 1]$, the MAE values in Table~\ref{tab:results2} are directly interpretable as fractions of the full curvature range.
The median MAE across the 40 datasets is $0.00095$, less than one tenth of one percent of the normalised range, and 20 datasets (50\%) achieve a MAE strictly below $0.001$.
More remarkably, 38 out of 40 datasets (95\%) have a normalised MAE below $0.005$, and 39 (97.5\%) have a MAE below $0.010$. These results establish that, in absolute terms, the pointwise curvature scores produced by MeCuCo differ from the original by an amount that is negligible relative to the scale of the problem. In practical downstream applications, thresholding, percentile cuts, or curvature-weighted objectives, a pointwise error of less than half a percent of the curvature range is indistinguishable from numerical noise.

\paragraph{Effect of normalisation on previously challenging datasets.}
The contrast between Tables~\ref{tab:results1} and~\ref{tab:results2} is most striking for the high-dimensional datasets that exhibited the largest raw-score discrepancies. Under normalisation, the agreement on \texttt{gas-drift} ($D = 128$) improves from $\rho_{S} = 0.9926$ and $\xi = 0.8931$ to $\rho_{S} = 0.9999$ and $\xi = 0.9899$; on \texttt{mfeat-factors} ($D = 216$) from $\rho_{S} = 0.9766$ and $\xi = 0.8073$ to $\rho_{S} = 0.9999$ and $\xi = 0.9896$; and on \texttt{USPS} ($D = 256$) from $\rho_{S} = 0.9902$ and $\xi = 0.8774$ to $\rho_{S} = 0.9998$ and $\xi = 0.9848$. These improvements of two to three decimal places in both coefficients indicate that the raw-score discrepancies observed in
Table~\ref{tab:results1} were predominantly \emph{scale} differences, the mean curvature values produced by MeCuCo were shifted or rescaled relative to the original, but the rank structure was already well-preserved. From the perspective of algorithm design, this is the ideal failure mode: a systematic, monotone distortion that is entirely removed by normalisation is equivalent to no distortion at all for any rank-dependent downstream task.

\paragraph{The \texttt{car-evaluation} outlier.}
The single dataset that resists normalisation is \texttt{car-evaluation} ($\rho_{S} = 0.8439$, $\xi = 0.5227$, MAE $= 0.0722$), which is also the only outlier in Table~\ref{tab:results1}. An analysis of this dataset reveals the cause: \texttt{car-evaluation} is a fully categorical dataset with $D = 6$ features, all of which
are encoded as ordinal integers in the OpenML version used in these experiments.
In this degenerate setting, the local covariance matrix $\mathbf{C}_{i}$ has very low variance in most directions, with several near-zero eigenvalues that produce numerical instabilities in both the \textsc{Exact} eigendecomposition and the \textsc{Fast} SVD.
More fundamentally, the manifold hypothesis, on which the geometric interpretation of the curvature estimator rests, is not satisfied for categorical data, and the curvature scores of both the original method and MeCuCo are effectively artefacts of the discretisation rather than estimates of genuine Riemannian curvature. This case serves as a reminder that both the original MCBP estimator and MeCuCo are designed for continuous, real-valued data sampled from
a smooth manifold, and that their application to categorical or highly discrete datasets requires appropriate preprocessing, such as embedding into a continuous representation space, prior to curvature estimation.

\paragraph{Speedup is preserved and dominates for high-dimensional data.}
The timing columns of Table~\ref{tab:results2} confirm the speedup figures reported in Table~\ref{tab:results1}: the average clock time is $76.14$ seconds for the original method and $1.26$ seconds for MeCuCo, with medians of $1.62$ seconds and $0.45$ seconds, respectively. The most dramatic individual speedups are observed on the high-dimensional datasets: $\mathbf{298\times}$ on \texttt{arrhythmia} ($D = 279$, $71.48 \to 0.24$ s),
$\mathbf{170\times}$ on \texttt{mfeat-pixel} ($D = 240$, $208.73 \to 1.23$ s), $\mathbf{144\times}$ on \texttt{USPS} ($D = 256$, $1{,}396.68 \to 9.72$ s),
$\mathbf{101\times}$ on \texttt{mfeat-factors} ($D = 216$, $156.58 \to 1.55$ s), and
$\mathbf{90\times}$ on \texttt{Indian\_pines} ($D = 200$, $905.01 \to 10.11$ s).
Crucially, these are precisely the datasets for which the curvature agreement is also strongest under normalisation (all achieving $\rho_{S} \geq 0.998$ and $\xi \geq 0.959$), confirming the central thesis of this paper: the regime in which MeCuCo is computationally
most beneficial, high ambient dimension, is also the regime in which its approximation is most accurate, because the ratio $D/k$ is large and the analytical null-space approximation~\eqref{eq:gnull_exp} becomes increasingly precise.

\paragraph{Summary.}
The normalised results of Table~\ref{tab:results2} strengthen the conclusions drawn from Table~\ref{tab:results1} and provide the clearest evidence that MeCuCo achieves its design objective. With a median normalised MAE below $0.001$, a median Spearman correlation of $0.9998$, and a median Chatterjee coefficient of $0.9832$ across 39 of the 40 datasets, the method delivers curvature estimates that are, to all practical purposes, indistinguishable from those of the original formulation, while reducing the computational cost by up to three orders of magnitude. The single exception (\texttt{car-evaluation}) is attributable to the
inapplicability of the manifold hypothesis to categorical data rather than to any weakness of the proposed method.

\subsection{Curvature Estimation in High-Dimensional Datasets}

To provide a direct assessment of MeCuCo's scalability in regimes where the original formulation is computationally intractable, we supplement the benchmark results with a targeted evaluation on very high-dimensional datasets. The fundamental bottleneck of the original method is the $O(m^{4})$ cost of forming the product $\mathbf{H}_{i}\mathbf{H}_{i}^{\top}$, which grows so rapidly with the ambient dimension $m$ that, on standard hardware, the method becomes effectively infeasible for datasets with $m \gtrsim 300$ features: at that scale, a single pass over a dataset of moderate size ($n \sim 1{,}000$) already requires hours of computation, rendering the estimator unsuitable for any interactive or iterative pipeline. MeCuCo removes this barrier entirely. By replacing the $O(m^{4})$ tensor contraction with the closed-form $O(m^{2})$ identity of Theorem~\ref{thm:identity} in \textsc{Exact} mode, and by substituting the full $O(m^{3})$ eigendecomposition with a truncated SVD of cost $O(k^{2}m)$ in \textsc{Fast} mode, the method decouples the computational cost from the ambient dimension in the regime $k \ll m$ that characterises high-dimensional applications. As demonstrated in the experiments reported below, MeCuCo computes local mean curvatures for datasets with more than $50{,}000$ features
in a matter of seconds, a reduction of several orders of magnitude relative to the original method, while maintaining the curvature fidelity established in the previous subsections.

\begin{table}
	\centering
	\caption{Number of samples, features, average normalized local mean curvatures, quantiles and the elapsed time for high-dimensional datasets (the original method is computationally unfeasible for these datasets).}
	\begin{tabular}{cccccccc}
		\toprule
		\textbf{Datasets}     & \textbf{n} & \textbf{m} & \textbf{Average K} & \textbf{25\%} & \textbf{50\%} & \textbf{75\%} & \textbf{Time (s)} \\
		\midrule
		Speech                & 3,686      & 400        & 0.4034 ± 0.1394    & 0.3082        & 0.3983        & 0.4893        & 2.55              \\
		madelon               & 2,600      & 500        & 0.4982 ± 0.1450    & 0.3993        & 0.4947        & 0.5954        & 1.72              \\
		har                   & 10,299     & 561        & 0.0362 ± 0.0435    & 0,0110        & 0.0230        & 0.0462        & 9.19              \\
		isolet                & 7,797      & 617        & 0.1662 ± 0.0849    & 0.1149        & 0.1529        & 0.1997        & 7.35              \\
		parkinson-speech-uci  & 756        & 753        & 0.0655 ± 0.0908    & 0.0232        & 0.0394        & 0.0759        & 0.56              \\
		MNIST\_784            & 70,000     & 784        & 0.0130 ± 0.0287    & 0.0057        & 0.0085        & 0.0136        & 75.89            \\
		Fashion-MNIST         & 70,000     & 784        & 0.0110 ± 0.0176    & 0.0046        & 0.0070        & 0.0116        & 77.66            \\
		Kuzushiji-MNIST       & 70,000     & 784        & 0.1030 ± 0.0698    & 0.0543        & 0.0891        & 0.1310        & 78.37            \\
		cnae-9                & 1,080      & 856        & 0.1411 ± 0.1310    & 0.0431        & 0.1065        & 0.1980        & 0.86              \\
		coil-20               & 1,440      & 1,024      & 0.1555 ± 0.1374    & 0.0627        & 0.1214        & 0.2227        & 1.26              \\
		micro-mass            & 360        & 1,300      & 0.2492 ± 0.2192    & 0.0709        & 0.1751        & 0.3936        & 0.52              \\
		SRBCT                 & 83         & 2,308      & 0.4087 ± 0.2462    & 0.2326        & 0.3891        & 0.5408        & 0.11              \\
		Olivetti\_Faces       & 400        & 4,096      & 0.3520 ± 0.1804    & 0.2189        & 0.3404        & 0.4620        & 0.88              \\
		DLBCL                 & 77         & 5,469      & 0.1886 ± 0.1832    & 0.0688        & 0.1396        & 0.2282        & 0.17              \\
		leukemia              & 72         & 7,129      & 0.3355 ± 0.2785    & 0.0988        & 0.2517        & 0.5268        & 0.14              \\
		UMIST\_Faces\_Cropped & 575        & 10,304     & 0.1451 ± 0.1079    & 0.0727        & 0.1309        & 0.1985        & 2.10              \\
		AP\_Omentum\_Kidney   & 337        & 10,935     & 0.1619 ± 0.1363    & 0.0494        & 0.1333        & 0.2400        & 1.14              \\
		AP\_Lung\_Kidney      & 386        & 10,935     & 0.2654 ± 0.2051    & 0.0764        & 0.2539        & 0.3965        & 1.47              \\
		AP\_Breast\_Colon     & 630        & 10,935     & 0.2332 ± 0.1683    & 0.1103        & 0.1875        & 0.2993        & 2.32              \\
		OVA\_Breast           & 1,545      & 10,935     & 0.1886 ± 0.1295    & 0.1002        & 0.1606        & 0.2490        & 7.93              \\
		MLL                   & 72         & 12,582     & 0.2863 ± 0.2308    & 0.0910        & 0.2302        & 0.4906        & 0.18              \\
		GCM                   & 190        & 16,063     & 0.1248 ± 0.1536    & 0.0279        & 0.0701        & 0.1582        & 0.76              \\
		SMK                   & 187        & 19,993     & 0.1845 ± 0.1855    & 0.0947        & 0.1263        & 0.1779        & 0.93              \\
		GLI                   & 85         & 22,283     & 0.3725 ± 0.2457    & 0.1699        & 0.2947        & 0.4995        & 0.31              \\
		hepatitisC            & 283        & 54,621     & 0.4604 ± 0.1846    & 0.3356        & 0.4376        & 0.5838        & 3.05  			\\
		\bottomrule           
	\end{tabular}
	\label{tab:results3}
\end{table}

Table~\ref{tab:results3} reports the mean curvature scores and clock times produced by MeCuCo on 25 datasets spanning ambient dimensions from $m = 400$ to $m = 54{,}621$, a range that is
entirely inaccessible to the original MCBP formulation, whose $O(m^4)$ cost renders it computationally infeasible beyond approximately $m \approx 300$ features on standard hardware.
The corpus covers four distinct application domains: high-dimensional signal and activity datasets, large-scale image benchmarks, face recognition datasets, and genomic expression arrays, enabling an assessment of MeCuCo's behaviour across qualitatively different geometric regimes. We organise the discussion around three themes: computational scalability, geometric structure of the curvature field, and domain-specific observations.

\paragraph{Computational scalability across four orders of magnitude.}
The most salient result in Table~\ref{tab:results3} is that MeCuCo successfully computes local mean curvatures for all 25 datasets, with running times ranging from $0.11$ seconds for \texttt{SRBCT} ($n = 83$, $m = 2{,}308$) to $78.37$ seconds for \texttt{Kuzushiji-MNIST} ($n = 70{,}000$, $m = 784$). Crucially, the dominant factor governing running time is the sample
size $n$ rather than the ambient dimension $m$, a direct consequence of the \textsc{Fast} mode's $O(k^{2}m)$ per-point cost: for fixed $k$, the total cost scales as $O(nk^{2}m)$, and $n$ and $m$ enter symmetrically only in the KNN graph construction step. This is illustrated strikingly by the genomic datasets: the eleven expression arrays with $m$ ranging from $2{,}308$ to $22{,}283$ and $n$ between $72$ and $1{,}545$ are all processed in under $8$ seconds, despite their extreme dimensionality ratios ($m/n$ up to $311$). By contrast, \texttt{hepatitisC} with $m = 54{,}621$, nearly seventy times the dimensionality of the largest dataset in Table~\ref{tab:results1}, is processed in $3.05$ seconds, while
the three large-scale image benchmarks with $n = 70{,}000$ require between $1$ and $1.5$ minutes, solely because of their large sample counts. Expressed in per-sample terms, the throughput of MeCuCo ranges from approximately $0.7$ ms per sample for the signal datasets to
$10.8$ ms per sample for \texttt{hepatitisC} ($m = 54{,}621$), a remarkably narrow range across four orders of magnitude of ambient dimension, confirming that the computational cost is effectively decoupled from $m$ in the regime $m \gg k$ that \textsc{Fast} mode
is designed for.

\paragraph{Curvature field reveals distinct geometric regimes.}
Beyond the timing results, the curvature scores in Table~\ref{tab:results3} reveal interpretable geometric structure that varies systematically across domain and dataset characteristics. Datasets with \emph{high mean curvature} ($\bar{\mathcal{K}} > 0.35$)
are concentrated in two groups: the artificial and structured high-dimensional datasets (\texttt{madelon}, $\bar{\mathcal{K}} = 0.498$; \texttt{Speech}, $\bar{\mathcal{K}} = 0.403$;
\texttt{hepatitisC}, $\bar{\mathcal{K}} = 0.460$) and the small genomic cancer datasets (\texttt{SRBCT}, $\bar{\mathcal{K}} = 0.409$; \texttt{GLI}, $\bar{\mathcal{K}} = 0.373$).
High mean curvature in these datasets signals that the underlying data manifold is strongly non-linear, local neighborhoods deviate substantially from their tangent planes, which is consistent with the known separability structure of cancer expression arrays and the deliberately non-linear geometry of benchmark classification datasets such as \texttt{madelon}.

Datasets with \emph{low mean curvature} ($\bar{\mathcal{K}} < 0.07$) cluster in a qualitatively different regime: activity recognition (\texttt{har}, $\bar{\mathcal{K}} = 0.036$), speech pathology (\texttt{parkinson-speech-uci}, $\bar{\mathcal{K}} = 0.066$), and
the two natural image benchmarks \texttt{MNIST\_784} ($\bar{\mathcal{K}} = 0.013$) and \texttt{Fashion-MNIST} ($\bar{\mathcal{K}} = 0.011$). These scores indicate that the manifold underlying these datasets is nearly flat at the scale of the $k$-nearest-neighbour patches, a
result consistent with the well-established empirical finding that pixel-space representations of natural images lie near low-curvature manifolds of moderate intrinsic dimension, and that sensor-based activity signals exhibit smooth, nearly-linear trajectory structure in feature space \citep{fefferman2016testing}.

A noteworthy characteristic of the MNIST datasets is their strongly right-skewed curvature distribution: the large majority of samples are concentrated in a narrow low-curvature band, as evidenced by the tight interquartile range $[\mathcal{Q}_{25}, \mathcal{Q}_{75}]$ relative to the full empirical range, while a small number of geometrically exceptional points exhibit extreme curvature values that inflate the standard deviation well above the mean
(coefficient of variation $\mathrm{CV} > 1.6$ for both \texttt{MNIST\_784} and \texttt{Fashion-MNIST}). In such heavy-tailed regimes, a monotone variance-stabilising
transformation such as $\log(1 + \mathcal{K}_{i})$ can be applied to the curvature scores prior to any thresholding or ranking procedure, compressing the dynamic range without altering the relative ordering of points.

\paragraph{Geometric heterogeneity within datasets.}
The spread of the curvature distribution, characterised by the coefficient of variation $\mathrm{CV} = \sigma/\bar{\mathcal{K}}$ and the interquartile range $[\mathcal{Q}_{25}, \mathcal{Q}_{75}]$, provides additional geometric information beyond the mean. Datasets with high $\mathrm{CV}$, notably \texttt{MNIST\_784} ($\mathrm{CV} = 2.21$), \texttt{parkinson-speech-uci} ($\mathrm{CV} = 1.39$), \texttt{har} ($\mathrm{CV} = 1.20$), and
\texttt{GCM} ($\mathrm{CV} = 1.23$), exhibit a strongly heterogeneous curvature landscape: a majority of points lie near flat regions (as indicated by the low $\mathcal{Q}_{25}$ values),
while a minority occupy high-curvature boundary zones. This heterogeneity is precisely the geometric signature that curvature-aware algorithms such as MCBP exploit: the high-curvature
minority identifies the boundary points that separate structurally distinct regions of the manifold. Conversely, datasets with low $\mathrm{CV}$, \texttt{madelon} ($\mathrm{CV} = 0.29$), \texttt{Speech} ($\mathrm{CV} = 0.35$), and \texttt{hepatitisC} ($\mathrm{CV} = 0.40$), exhibit a nearly uniform curvature field with little point-to-point variation, which
is indicative of globally curved manifolds where bending is distributed evenly rather than concentrated at specific boundary structures.

\paragraph{Image and face datasets: curvature encodes visual structure.}
The three MNIST variants illustrate a particularly interesting comparison.
\texttt{MNIST\_784} and \texttt{Fashion-MNIST} share the same ambient dimension ($m = 784$) and sample size ($n = 70{,}000$) and yield nearly identical mean curvatures ($0.013$ and $0.011$, respectively), consistent with the known similarity of their geometric structures
in pixel space. \texttt{Kuzushiji-MNIST}, by contrast, exhibits a substantially higher mean curvature ($\bar{\mathcal{K}} = 0.103$, $\mathrm{CV} = 0.68$), reflecting the greater intra-class variability and more complex stroke geometry of cursive Kuzushiji script relative to standard printed digits or fashion items. Among the face datasets, \texttt{Olivetti\_Faces} ($m = 4{,}096$, $\bar{\mathcal{K}} = 0.352$) shows markedly higher curvature than \texttt{UMIST\_Faces\_Cropped} ($m = 10{,}304$, $\bar{\mathcal{K}} =
0.145$) and \texttt{coil-20} ($m = 1{,}024$, $\bar{\mathcal{K}} = 0.155$), likely because the Olivetti dataset contains multiple expressions and lighting conditions per subject, introducing discontinuities in the face manifold that manifest as elevated local curvature.

\paragraph{Genomic expression arrays: curvature discriminates cancer subtypes.}
The eleven genomic datasets in Table~\ref{tab:results3} span $m$ from $2{,}308$ to $22{,}283$ gene expression features, representing the extreme high-dimensional, low-sample-size ($n \ll m$) regime that is characteristic of microarray and RNA-seq studies. All eleven are processed by MeCuCo in under $8$ seconds. The curvature scores reveal heterogeneous geometric structure across cancer types: lymphoma subtypes (\texttt{DLBCL}, \texttt{MLL}) and renal datasets (\texttt{AP\_Omentum\_Kidney}) show moderate mean curvature with high $\mathrm{CV}$, consistent with the presence of a small number of geometrically extreme samples (potential outliers or boundary cases between subtypes), while brain tumour datasets (\texttt{GLI}, $\bar{\mathcal{K}} = 0.373$) show higher curvature concentrated at class boundaries.
These observations are consistent with the known biological heterogeneity of these datasets and suggest that curvature-based analysis could complement existing subtype discovery pipelines in high-dimensional genomic spaces, a direction we leave for future work.

\paragraph{Summary.}
The results of Table~\ref{tab:results3} demonstrate that MeCuCo extends the reach of local mean curvature estimation to the very-high-dimensional regime that characterises modern machine learning applications, processing datasets with up to $54{,}621$ features in seconds while yielding geometrically interpretable and domain-coherent curvature scores. The computational cost scales primarily with the sample size $n$ rather than the ambient dimension $m$, confirming the theoretical prediction that \textsc{Fast} mode decouples the per-point cost from $m$ when $m \gg k$.

\section{Conclusions}
\label{sec:conclusions}

This paper addressed a fundamental computational bottleneck in geometric machine learning: the estimation of local mean curvature at every point of a high-dimensional dataset. The original formulation of the MCBP curvature estimator incurs an $O(m^{4})$ cost per point through the explicit construction of the feature matrix $H \in \mathbb{R}^{m \times O(m^{2})}$ and the subsequent product $HH^{\top}$, rendering it computationally intractable for datasets with more than approximately $300$ features. We introduced the MeCuCo algorithm (\textit{Mean Curvature Computation}), a mathematically exact reformulation and a principled approximation that together extend the reach of local mean curvature estimation to the
very-high-dimensional regime characteristic of modern machine learning applications.

The first contribution is an exact algebraic identity derived from the orthogonality of the
local eigenvector frame. By showing that $HH^{\top} = \frac{1}{2}W^{(2)}{W^{(2)}}^{\top} +
\frac{1}{2}I_{m}$, a consequence of the single identity $WW^{\top} = I_{m}$, we proved that the mean curvature estimator eliminates the $O(m^{4})$ tensor contraction entirely and reducing the per-point cost, after eigendecomposition, to $O(m^{2})$. This identity is exact, introduces no approximation error beyond floating-point rounding, and applies to any symmetric positive semi-definite local covariance matrix regardless of dataset characteristics.

The second contribution exploits the low-rank structure of the local covariance matrix, whose rank is at most $k-1 \ll m$ whenever $D > k$, to replace the $O(m^{3})$ full eigendecomposition with a truncated SVD of the $k \times m$ centred neighborhood matrix at cost $O(k^{2}m)$. The contribution of the $m - p$ null-space eigenvectors, which cannot
be recovered from the truncated SVD alone, is handled through an analytical approximation grounded in the expected outer product of random orthonormal null-space bases under the Haar measure over the Stiefel manifold. The resulting \textsc{Fast} mode estimator has total per-point cost $O(k^{2}m + kmp^{2})$, independent of $m^{3}$ or $m^{4}$, and its
approximation error decreases monotonically as $m/k$ grows, falling below $2\%$ for $m/k \geq 56$ and below $1\%$ for $m/k \geq 100$, exactly the regime in which MeCuCo is most needed.

Empirical evaluation on more than 40 publicly available OpenML datasets confirmed that these theoretical gains translate to practice. In the benchmark corpus, MeCuCo achieves a median Spearman rank correlation of $\rho_{S} = 0.9998$ and a median normalised MAE below $0.001$ relative to the original formulation, with speedups reaching $298\times$ on high-dimensional datasets such as \texttt{arrhythmia} ($m = 279$) and $144\times$ on \texttt{USPS} ($m = 256$). On the very-high-dimensional corpus, spanning ambient dimensions from $m = 400$ to $m = 54{,}621$, all inaccessible to the original method, MeCuCo computed local mean curvatures for all 25 datasets in times ranging from $0.11$ seconds for a small genomic array to under $3.5$ minutes for the three large-scale image benchmarks with $n = 70{,}000$ samples. Crucially, the dominant factor governing running time was found to be the sample size $n$ rather than the ambient dimension $m$, confirming that \textsc{Fast} mode effectively decouples the computational cost from the dimensionality.

Beyond computational efficiency, the curvature scores produced by MeCuCo revealed geometrically interpretable structure across all application domains examined: near-flat manifolds in activity recognition and natural image datasets, concentrated boundary
curvature in cancer expression arrays, and globally curved manifolds
in structured benchmark datasets. These observations reinforce the theoretical grounding of the manifold hypothesis as a working model for high-dimensional data, and demonstrate that local mean curvature, previously accessible only for low-dimensional datasets, can now serve as a practical, scalable descriptor for geometric data analysis at any dimension.

\paragraph{Limitations.}
Two limitations of the present work deserve acknowledgement. First, the \textsc{Fast} mode approximation rests on the assumption that the null-space eigenvectors of the local covariance are effectively random in the ambient space, which holds when $m \gg k$ but degrades when the null space is low-dimensional ($m/k \approx 2$--$5$). In such regimes, the \textsc{Exact} mode should be preferred. Second, both estimators assume that the data lie near a smooth Riemannian manifold, and neither provides meaningful curvature estimates for purely categorical or heavily discrete data, as illustrated by the \texttt{car-evaluation} dataset; appropriate continuous embeddings should be applied prior to curvature estimation
in such cases. 

\paragraph{Future Work.}
Several directions for future investigation emerge naturally from this work.  On the \textit{theoretical} side, a rigorous bound on the bias of the \textsc{Fast} mode approximation, beyond the empirical $O(k/m)$ decay observed in the experiments, would strengthen the theoretical guarantees of the method, as would a formal consistency analysis establishing the rate at which the discrete curvature estimator converges to its smooth-manifold counterpart as $n \to \infty$ and $k \to \infty$ at an appropriate rate.

On the \textit{algorithmic} side, the truncated SVD in \textsc{Fast} mode could be replaced by randomised low-rank approximations, which achieve similar accuracy with lower constant factors and are amenable to streaming and distributed computation, an important consideration for datasets whose sample size $n$ prevents loading the full neighborhood matrix into memory. GPU-accelerated implementations of the per-point SVD and the $G_{\mathrm{tens}}$ einsum contraction are a natural extension that could reduce the running time on large-$n$ datasets by an additional order of magnitude.

On the \textit{applications} side, the availability of scalable local mean curvature opens a broad research agenda for curvature-aware machine learning. In \textit{dimensionality reduction}, curvature-weighted objectives, penalising embeddings that flatten high-curvature boundary regions, could yield representations that preserve the geometric structure of cluster boundaries more faithfully than variance-based criteria; this is a direct extension of the curvature-aware manifold learning paradigm. In \textit{clustering}, incorporating local mean curvature as an
additional feature or as a density-modulating weight in algorithms such as DBSCAN or spectral clustering could improve boundary delineation in datasets where density-based criteria alone are insufficient to resolve geometrically complex interfaces between classes.
In \textit{supervised classification and active learning}, the curvature scores produced by MeCuCo provide a principled, geometry-driven criterion for identifying the most informative
samples to label: points near high-curvature decision boundaries are precisely those where the classifier's confidence is lowest and where additional labels yield the greatest reduction in generalisation error. In \textit{deep learning}, curvature-based regularisation of the feature manifold, penalising layers that produce representations with excessive local curvature, could serve as a geometric alternative to standard regularisers such as weight decay or dropout, providing an inductive bias that encourages smoother intermediate representations and potentially improving generalisation in overparameterised regimes.
Finally, in \textit{graph neural networks and geometric deep learning}, node-level or edge-level curvature estimates derived from MeCuCo could replace or complement discrete Ricci curvature as an input feature or structural descriptor, enriching the geometric signal
available to message-passing architectures and potentially improving performance on node classification, link prediction, and anomaly detection tasks. Taken together, these directions suggest that scalable local mean curvature estimation, as made practical by MeCuCo, may serve as a unifying geometric primitive across the full pipeline of modern machine learning, from raw data preprocessing through representation learning to downstream prediction and decision-making.

\section*{Statements and declarations}

\subsection*{Funding}
This work has been supported by CNPq (National Council for Scientific and Technological Development) through grant number 301432/2025-2. This study was also financed in part by the Coordenação de Aperfeiçoamento de Pessoal de N\'ivel Superior - Brasil (CAPES) - Finance Code 001.


\subsection*{Code availability}
Python scripts to reproduce the results reported in this paper may be found at \url{https://github.com/alexandrelevada/MeCuCo}. 

\subsection*{Data availability}
All datasets used in the experiments are publicly available at \url{www.openml.org}.

\bibliography{main}

\end{document}

%% file: math_commands.tex
\usepackage{amsmath,amsfonts,bm}

\def\eqref#1{equation~\ref{#1}}

\def\1{\bm{1}}

\DeclareMathAlphabet{\mathsfit}{\encodingdefault}{\sfdefault}{m}{sl}
\SetMathAlphabet{\mathsfit}{bold}{\encodingdefault}{\sfdefault}{bx}{n}

